\documentclass[10pt]{article} 

\usepackage[accepted]{rlc}

\usepackage{authblk}
\usepackage{dashrule}

\usepackage{times}
\usepackage[utf8]{inputenc} 
\usepackage[T1]{fontenc}    
\usepackage{url}            
\usepackage{booktabs}       
\usepackage{nicefrac}       
\usepackage{microtype}      
\usepackage{graphics,color}

\usepackage{algorithm}
\usepackage{algorithmic}
\usepackage{enumitem}

\usepackage{amsfonts}       
\usepackage{amsmath}       
\usepackage{amssymb}

\usepackage{xspace}
\makeatletter
\DeclareRobustCommand\onedot{\futurelet\@let@token\@onedot}
\def\@onedot{\ifx\@let@token.\else.\null\fi\xspace}
\makeatother

\usepackage{xr-hyper}
\makeatletter
\newcommand*{\addFileDependency}[1]{
  \typeout{(#1)}
  \@addtofilelist{#1}
  \IfFileExists{#1}{}{\typeout{No file #1.}}
}
\makeatother

\usepackage{xcolor}
\definecolor{ourblue}{rgb}{0.368,0.507,0.71}
\definecolor{ourorange}{rgb}{0.881,0.611,0.142}
\definecolor{ourgreen}{rgb}{0.56,0.692,0.195}
\definecolor{ourred}{rgb}{0.923,0.386,0.209}
\definecolor{ourviolet}{rgb}{0.528,0.471,0.701}
\definecolor{ourbrown}{rgb}{0.772,0.432,0.102}
\definecolor{ourlightblue}{rgb}{0.364,0.619,0.782}
\definecolor{ourdarkolive}{rgb}{0.572,0.586,0.}
\definecolor{ourpurple}{rgb}{0.528,0.471,0.701}

\definecolor{ourdarkbrown}{rgb}{0.618,0.348,0.0816}
\definecolor{ourcyan}{rgb}{0.364,0.619,0.782}
\definecolor{ourgrey}{rgb}{0.75,0.75,0.75}

\definecolor{ourdarkred}{rgb}{0.67, 0.22, 0.07}
\definecolor{ourdarkorange}{rgb}{0.71, 0.49, 0.1}
\definecolor{ourdarkblue}{rgb}{0.27, 0.4, 0.58}
\definecolor{ourdarkgreen}{rgb}{0.41, 0.51, 0.15}

\definecolor{ourcyan2}{rgb}{0.125,0.722,0.804}
\definecolor{ourred2}{rgb}{0.863,0.184,0.047}
\definecolor{ouryellow2}{cmyk}{0,0.16,1.0,0.07}
\definecolor{ourviolet2}{cmyk}{0.55,0.56,0,0.47}
\definecolor{ourorange2}{cmyk}{0,0.46,0.89,0.11}

\usepackage{tikz}
\usetikzlibrary{decorations.markings}
\newcommand{\hdottedrule}[3][0]{%
	\tikz[baseline]{\path[decoration={markings,
			mark=between positions 0 and 1 step 2*#3
			with {\node[fill, circle, minimum width=#3, inner sep=0pt, anchor=south west] {};}},postaction={decorate}]  (0,#1) -- ++(#2,0);}}
\newcommand{\hdashdotrule}[3][0]{
	\tikz[baseline=-#1]\draw[thick,dash dot, line width=0.2em] (0,0) -- (#2,0);
}

\usepackage{IEEEtrantools}
\usepackage{multirow}
\usepackage{multicol}

\usepackage{amsmath}
\usepackage{amsthm}
\usepackage{bbm}
\usepackage{algorithm}
\usepackage{cuted}
\usepackage{breqn}
\usepackage{booktabs}       %
\usepackage{nicefrac}       %
\usepackage{microtype}      %

\usepackage[utf8]{inputenc} 
\usepackage[T1]{fontenc}    
\usepackage{hyperref}       
\usepackage{cleveref}
\usepackage{url}            
\usepackage{booktabs}       
\usepackage{nicefrac}       
\usepackage{microtype}      
\usepackage{xcolor}         
\usepackage{DejaVuSans}
\usepackage{subcaption}
\usepackage{svg}
\usepackage{wrapfig}
\usepackage{mathtools}

\definecolor{blue}{HTML}{4d71a6}
\definecolor{green}{HTML}{2e7647}
\definecolor{brown}{HTML}{6d5959}
\definecolor{orange}{HTML}{DE9102}
\definecolor{red}{HTML}{ff4e33}

\definecolor{sc1}{HTML}{aec0da}
\definecolor{sc2}{HTML}{7d9ac4}
\definecolor{sc3}{HTML}{4f75ac}
\definecolor{sc4}{HTML}{38547b}

\crefname{section}{Section}{Sections}
\crefname{appendix}{Supplementary}{Supplementary}
\crefname{figure}{Figure}{Figures}
\crefname{lem}{Lemma}{Lemmas}
\crefname{thm}{Theorem}{Theorems}
\crefname{asm}{Assumption}{Assumptions}
\crefname{cor}{Corrolary}{Corrolaries}

\crefname{thm}{Thm.}{Thms.}
\crefname{lem}{Lem.}{Lems.}
\crefname{cor}{Cor.}{Cors.}
\crefname{def}{Def.}{Defs.}
\crefname{exmp}{Ex.}{Exs.}
\crefname{section}{Sec.}{Secs.}
\crefname{subsection}{Subsec.}{Subsecs.}
\crefname{appendix}{Supp.}{Supps.}
\crefname{asm}{Assum.}{Assum.}

\definecolor{doi0}{HTML}{B2C4DB}
\definecolor{doi1}{rgb}{0.2982297551789076, 0.4433145009416194, 0.6528813559322034}
\definecolor{doi2}{rgb}{0.4494117647058823, 0.5717647058823528, 0.7484967320261439}
\definecolor{doi4}{rgb}{0.6188235294117648, 0.7035294117647058, 0.8258823529411765}
\definecolor{ouryellow}{HTML}{e09b24}

\def\legendfont{\scriptsize}
\newlength{\legendw}
\newlength{\barw}
\newcommand{\legend}{\legendfont
	{\color{doi4}\hdashrule[.5ex]{\legendw}{.2em}{}} \method{}$^4$ \quad 
	{\color{doi2}\hdottedrule[.5ex]{\legendw}{.2em}} \method{}$^2$ \quad 
	{\color{doi1}\hdashrule[.5ex]{\legendw}{.2em}{2pt 0.5pt}} \method{}$^1$ \quad 
	{\color{ourorange}\hdashrule[.5ex]{\legendw}{.2em}{}} SMODICE$^\dagger$\\ 
}
\newcommand{\legendtwo}{\legendfont
	{\color{doi4}\hdashrule[.5ex]{\legendw}{.2em}{}} \method{}$^4$ \quad 
	{\color{doi2}\hdottedrule[.5ex]{\legendw}{.2em}} \method{}$^2$ \quad
	{\color{doi1}\hdashrule[.5ex]{\legendw}{.2em}{2pt 0.5pt}} \method{}$^1$\quad 
	{\color{ourorange}\hdashrule[.5ex]{\legendw}{.2em}{}} SMODICE$^\dagger$\\ 
}
\definecolor{f1}{HTML}{9eb3d3}
\definecolor{f2}{HTML}{7d9ac4}
\definecolor{f3}{HTML}{5d81b5}
\definecolor{f4}{HTML}{476a9c}
\definecolor{f5}{HTML}{e09b24}

\newcommand{\legendfixed}{\legendfont
	{\color{f1}\hdashrule[.5ex]{\legendw}{.2em}{}} $0.0$ \quad 
	{\color{f2}\hdashrule[.5ex]{\legendw}{.2em}{4pt 1.5pt}} $0.4$ \quad 
	{\color{f3}\hdottedrule[.5ex]{\legendw}{.2em}}  $0.2$ \quad
	{\color{f4}\hdashdotrule[.7ex]{\legendw}{.2em}{}}  $0.8$ \quad 
	{\color{f5}\hdashrule[.5ex]{\legendw}{.2em}{}}  $1.0$ 
}

\newcommand{\legendskills}{\legendfont
	{\color{ourblue}\hdashrule[.5ex]{0.3em}{.2em}{}}
	{\color{ouryellow}\hdashrule[.5ex]{0.3em}{.2em}{}}
	{\color{ourgreen}\hdashrule[.5ex]{0.3em}{.2em}{}}
	{\color{ourviolet}\hdashrule[.5ex]{0.3em}{.2em}{}}
	{\color{ourred}\hdashrule[.5ex]{0.3em}
		{.2em}{}} \quad Skills \\
}
\newcommand{\legenddrl}{\legendfont
	{\color{doi4}\hdashrule[.5ex]{\legendw}{.2em}{}} \method{}$^4$ \quad 
	{\color{doi2}\hdottedrule[.5ex]{\legendw}{.2em}} \method{}$^2$ \quad 
	{\color{doi1}\hdashrule[.5ex]{\legendw}{.2em}{2pt 0.5pt}} \method{}$^1$ \quad 
	{\color{ourdarkblue}\hdashrule[.5ex]{\legendw}{.2em}{2pt 0.5pt 1pt 0.5pt}} \method{}$^{.5}$ \quad 
 {\color{ourorange}\hdashrule[.5ex]{\legendw}{.2em}{}} SMODICE$^\dagger$\\ 
}

\DeclareSymbolFont{extraup}{U}{zavm}{m}{n}
\DeclareMathSymbol{\varheart}{\mathalpha}{extraup}{86}
\DeclareMathSymbol{\vardiamond}{\mathalpha}{extraup}{87}

\newlength{\legendwenvs}
\newcommand{\legendenvs}{\legendfont
{\color{black}  $\clubsuit$} Walker2D
\hspace{\legendwenvs} {\color{ourred} $\varheart$}  HalfCheetah 
\hspace{\legendwenvs} {\color{ourred} $\vardiamond$} Hopper 
\hspace{\legendwenvs} {\color{black} $\spadesuit$} Ant

}

\definecolor{altcolor}{RGB}{104, 149, 158}

\hypersetup{
	colorlinks=true,       
	linkcolor=ourdarkgreen,        
	citecolor=ourdarkgreen,        
	filecolor=ourdarkgreen,     
	urlcolor=ourdarkgreen         
}

\usepackage{amsmath,amsfonts,bm}

\def\1{\bm{1}}

\DeclareMathAlphabet{\mathsfit}{\encodingdefault}{\sfdefault}{m}{sl}
\SetMathAlphabet{\mathsfit}{bold}{\encodingdefault}{\sfdefault}{bx}{n}

\def\gD{{\mathcal{D}}}

\def\gI{{\mathcal{I}}}

\def\gO{{\mathcal{O}}}

\def\gS{{\mathcal{S}}}

\def\sR{{\mathbb{R}}}

\newcommand{\E}{\mathbb{E}}

\DeclareMathOperator*{\argmin}{arg\,min}

\newcommand{\x}{\mathbf{x}}

\theoremstyle{plain}
\newtheorem{thm}{Theorem}[section]
\newtheorem{lem}[thm]{Lemma}

\newtheorem{asm}[thm]{Assumption}

\theoremstyle{definition}

\theoremstyle{remark}

\newcommand{\method}{MyMethod\xspace}

\newcommand{\Df}{\mathrm{D}_f}
    \newcommand{\Dkl}{\mathrm{D}_\mathrm{KL}}

\newcommand{\numSeeds}{3\,}

\newcommand{\etaz}{\eta_{z}(s,a)}
\newcommand{\etae}{\eta_{\widetilde{E}}}
\newcommand{\dE}{d_{\widetilde{E}}}

\newcommand{\envhalfcheetah}{\textsc{HalfCheetah}}
\newcommand{\envwalker}{\textsc{Walker2D}}
\newcommand{\envhopper}{\textsc{Hopper}}
\newcommand{\envant}{\textsc{Ant}}
\newcommand{\envsolo}{\textsc{Solo12}}
\renewcommand{\method}{DOI}

\newif\ifcomments
\commentstrue 

\definecolor{darkblue}{HTML}{228BBB}
\definecolor{darkorange}{HTML}{E39F40}
\newif\ifmcolor
\mcolortrue 
\newcommand{\cdiv}[1]{\ifmcolor{{\color{ourdarkblue} #1}}\else#1\fi}
\newcommand{\cdem}[1]{\ifmcolor{{\color{ourdarkorange} #1}}\else#1\fi}

\definecolor{purple}{RGB}{128,0,128}

\author[1,2]{\textbf{Marin Vlastelica}}
\author[1,]{\textbf{Jin Cheng}}
\author[1,2]{\textbf{Georg Martius}}
\author[1,2]{\textbf{Pavel Kolev}}
\affil[1]{Max Planck Institute for Intelligent Systems, Tübingen, Germany}
\affil[2]{University of Tübingen, Tübingen, Germany}
\affil[3]{ETH Zurich, Switzerland}

\title{Offline Diversity Maximization\\ Under Imitation Constraints}

\begin{document}

\maketitle

\vskip 0.3in

\begin{abstract}
    There has been significant recent progress in the area of unsupervised skill discovery, utilizing various information-theoretic objectives as measures of diversity.
    Despite these advances, challenges remain: current methods require significant online interaction, fail to leverage vast amounts of available task-agnostic data and typically lack a quantitative measure of skill utility.
    We address these challenges by proposing a principled offline algorithm for unsupervised skill discovery that, in addition to maximizing diversity, ensures that each learned skill imitates state-only expert demonstrations to a certain degree.
    Our main analytical contribution is to connect Fenchel duality, reinforcement learning, and unsupervised skill discovery to maximize a mutual information objective subject to KL-divergence state occupancy constraints.
    Furthermore, we demonstrate the effectiveness of our method on the standard offline benchmark D4RL and on a custom offline dataset collected from a 12-DoF quadruped robot for which the policies trained in simulation transfer well to the real robotic system.\footnote{Project website with videos: \href{https://sites.google.com/view/diversity-via-duality/home}{https://tinyurl.com/diversity-via-duality}}
\end{abstract}


\begin{figure*}[htbp]
    \centering
    \includegraphics[width=.9\linewidth]{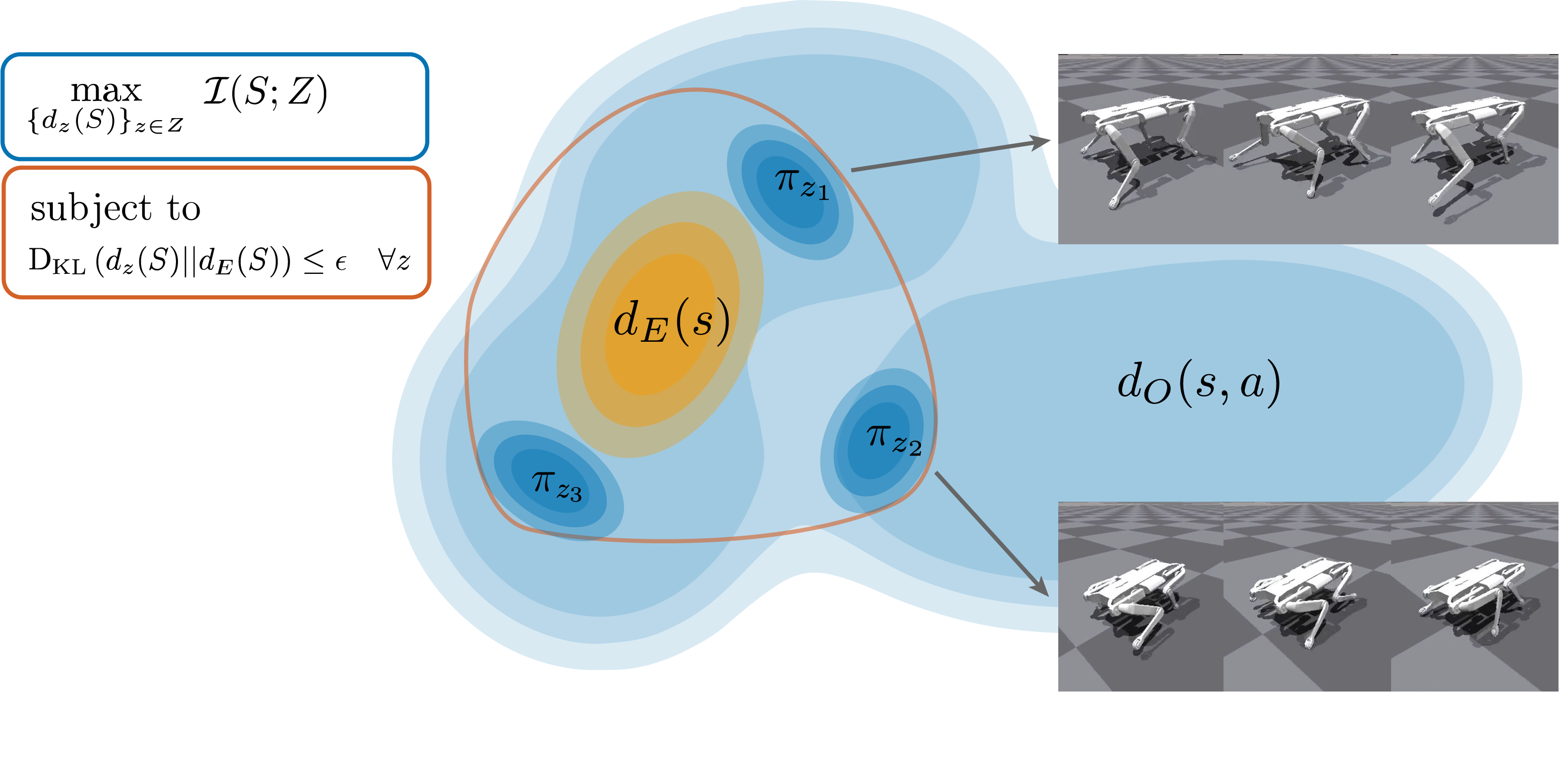}
    \vspace{-1em}
    \caption{{\small Diverse Offline Imitation (DOI) maximizes a variational lower bound on the mutual information between latent skills $z$ and states $s$ visited by associated skill-conditioned policies $\pi_z$, subject to a KL-divergence constraint to limit the deviation of the state occupancy $d_z(s)$ of each latent skill $z$ from that of an expert $d_E(s)$.}
    }
    \label{fig:main-figure}
\end{figure*}

\section{Introduction}

Recent advancements in reinforcement learning (RL) have included substantial progress in unsupervised skill discovery, aiming to empower autonomous agents with the capability to acquire a diverse set of skills directly from their environment, without relying on predefined human-engineered rewards or demonstrations.
These methods have the potential to revolutionize the way RL agents learn to solve complex tasks.
The growing interest in unsupervised skill discovery has led to various approaches, typically rooted in information-theoretic concepts, including empowerment~\citep{klyubin2005empowerment,mohamed2015variational,eysenbach19diayn}, information bottleneck~\citep{tishby99information,Goyal2019infobot,Kim2021IBOL} and information gain~\citep{houthooft2016vime,strouse2021disdain,parkL23dadsInfoGain}.
Despite these advancements, there remains a significant challenge. 
Current methods demand substantial online interaction with the environment, making exploration in high-dimensional state-action spaces inefficient.
Although~\citet{zahavy2023domino} introduced constraints to enhance skill performance and narrow the exploration space by incentivizing diverse skills to meet a certain utility measure, their approach does not eliminate the need for considerable online interaction with the environment.
Meanwhile, there have been significant recent advances in large-scale data collection \citep{RoboHive2020,walke2023bridgedata,rt2} and in the development of scalable and sample-efficient offline RL algorithms that leverage diverse behaviors of pre-collected experience.
However, these approaches struggle with well-known challenges, including off-policy evaluation and the out-of-distribution problem, which have been studied extensively in previous work~\citep{Levine2020SurveyRL, Prudencio2022SurveyRL}.

In this work, we address the aforementioned challenges by introducing a novel problem formulation and complementing it with the first principled \textit{offline} RL algorithm for unsupervised skill discovery that, in addition to maximizing diversity, ensures that each learned skill imitates state-only expert demonstrations to a certain degree.
More specifically, we consider a problem formulation with two datasets: a large one with diverse state-action demonstrations and another much smaller one with state-only expert demonstrations.
This setting is particularly valuable in robotics scenarios where expert demonstrations are limited and the domain of the expert may be different from that of the agent, such as in human demonstrations.
Another potential application is to enhance the realism of computer games by creating an immersive experience of interacting with non-player characters, each behaving in a slightly different style, while all partially imitating the behavior of a human expert.

We formulate the problem as a Constrained Markov Decision Process (CMDP)~\citep{altman1999constrained, cmdpblog} that seeks to maximize diversity through a mutual information objective, subject to Kullback-Leibler (KL) divergence state occupancy constraints ensuring that each skill imitates state expert demonstrations to a certain degree.
The resulting CMDP has convex objective and constraints, making the optimization problem intractable. 
We adopt a tractable relaxation approach consisting of an alternating scheme that maximizes a variational lower bound on mutual information, and to handle the constraints it applies Lagrange relaxation.
Our method, Diverse Offline Imitation (DOI), overcomes the off-policy evaluation by leveraging the Fenchel-Rockafellar duality in RL~\citep{nachum2020reinforcement,kim2022demodice,ma2022smodice} to connect a dual optimal value solution (computed using offline samples) with primal optimal state-action occupancy ratios.
These ratios serve as importance weights for offline training of a skill-conditioned policy, skill-discriminator, KL-divergence estimators, and Lagrange multipliers.
We demonstrate the effectiveness of our method on the standard offline benchmark D4RL~\citep{fu2020d4rl} and on a custom offline dataset collected from a 12-DoF quadruped robot Solo12~\citep{leziart2021implementation}.
In addition, we show that \method{} trained on simulation data transfers well to a real robot system.

\section{Related Work}

In the context of skill discovery \citet{achiam2018valor} and \citet{campos2020edl} showed that methods like DIAYN~\citep{eysenbach19diayn} can struggle to learn large numbers of skills and have a poor coverage of the state space. 
\citet{strouse2021disdain} observed that when a novel state is visited, the discriminator lacks sufficient training data to accurately classify skills, which results in a low intrinsic reward for exploration.
They address this by introducing an information gain objective (involving an ensemble of discriminators) as a bonus term.
\citet{kimPK21ibol} gave a skill discovery approach based on an information bottleneck that leads to disentangled and interpretable skill representations.
\citet{parkCKLK22lsd, parkLLA22newLSD} proposed a Lipschitz-constrained skill discovery method based on a distance-maximizing and controllability-aware distance function to overcome the bias toward static skills and to allow the agent to learn complex and far-reaching behaviors. 
\citet{sharma2019dads} developed a method that simultaneously discovers predictable skills and learns their dynamics.
In a follow-up work, \citet{parkL23dadsInfoGain} addresses the problem of errors in predictive models by learning a transformed MDP, whose action space contains only easy to model and predictable actions.
These works provide RL algorithms for unsupervised skill discovery that require \emph{online} interaction with the environment and do not impose utility measures on the learned skills.
In contrast, \method{} gives a principled \emph{offline} algorithm for maximizing diversity under imitation constraints.

A large body of research has focused on successor features~\citep{dayan1993successor, barreto2016successor}, a powerful technique in RL for transfer of knowledge across tasks by capturing environmental dynamics, particularly promising for skill discovery when coupled with variational intrinsic motivation~\citep{gregor2016vic,barreto2018transfer,hansen20visr} to enhance feature controllability, generalization, and task inference.
In contrast to our work, these approaches do not impose performance constraints on the learned skills.
\citet{zahavy2023domino} cast the task of learning diverse skills, each achieving a near-optimal performance with respect to a given reward, into a constrained MDP setting with a physics-inspired diversity objective based on a minimum $\ell_2$ distance between the successor features of distinct skills. 
However, this approach requires significant \emph{online} interaction with the environment to learn the skills.

Numerous practical algorithms for offline RL have been proposed~\citep{Levine2020SurveyRL, Prudencio2022SurveyRL}, including methods based on advantage-weighted behavioral cloning~\citep{nair2020awac,wang2020critic}, conservative strategies to stay close to the original data distribution~\citep{kumar2020conservative,cheng2022atac} and using only on-data samples~\citep{kostrikov2021iql,xu2023sql}.
While these methods excel at learning a policy that maximizes a fixed reward, they are not directly applicable in our setting, which has a non-stationary reward that depends on: i) the log-likelihood of a skill discriminator, and ii) Lagrange multipliers.
In addition, these techniques cannot be used to i) train a skill discriminator and ii) estimate a KL divergence offline.

Naive importance sampling approaches for off-policy estimation are known to suffer from unbounded variance in the infinite horizon setting, a problem known in the literature as ``the curse of horizon''.
\citet{Liu18breakingCurseOfHorizonOffPolicy,Mousavi2020offPolicyEstimation} addressed this challenge by providing theoretical foundations and a principled off-policy algorithm, using a backward Bellman operator, that avoids exploding variance by applying importance sampling to state-visitation distributions, and by providing practical solutions in Reproducing Kernel Hilbert Spaces.
An alternative research direction in off-policy estimation, referred to as ``Distribution Correction Estimation (DICE)'', has introduced innovative techniques, with \citet{nachum2019dualdice} mitigating variance with importance sampling, \citet{nachum2019algaedice} enabling policy gradient from off-policy data without importance weighting, \citet{kim2022demodice} stabilizing offline imitation learning with imperfect demonstrations, \citet{zhang2020gradientdice} improving density ratio estimation, \citet{dai2020coindice} providing high-confidence off-policy evaluation.
Subsequently, \citet{Xu2021SBAC} applied this approach to offline RL and demonstrated its effectiveness in continuous control tasks.
Our work uses a DICE-based off-policy approach similar to OptiDICE~\citep{lee2021optidice,lee2022coptidice} for estimating importance ratios, while considering a constrained formulation with a mutual information objective and KL-divergence imitation constraints.

\section{Preliminaries}

We utilize the framework of Markov decision processes (MDPs)~\citep{puterman2014markov}, where an MDP is defined by the tuple
$(\mathcal{S}, \mathcal{A}, \mathcal{R}, \mathcal{P}, \rho_0, \gamma)$ denoting the state space, action space, reward mapping $\mathcal{R}: \mathcal{S}\times \mathcal{A} \mapsto \sR$,
stochastic transition kernel $\mathcal{P}(s^{\prime}|s,a)$, initial state distribution $\rho_0(s)$ and discount factor $\gamma$.
A policy $\pi: \mathcal{S} \mapsto \Delta(\mathcal{A})$ defines a probability distribution over the action space $\mathcal{A}$ conditioned on the state, where $\Delta(\cdot)$ stands for the probability simplex.

Given a policy $\pi$, the corresponding state-action occupancy measure is defined by
\[
d^{\pi}(s,a):=(1-\gamma)\sum_{t=0}^{\infty}\gamma^{t}\mathrm{Pr}\big[s_{t}=s,a_{t}=a \,|\, s_{0}\sim\rho_{0} ,a_{t}\sim\pi(\cdot|s_{t}),s_{t+1}\sim \mathcal{P}(\cdot|s_{t},a_{t})\big]
\]
and its associated state occupancy $d^{\pi}(s)$ is given by marginalizing over the action space $\sum_{a\in\mathcal{A}} d^{\pi}(s,a)$.

In the skill discovery setting, $z\sim p(Z)$ denotes a fixed latent skill on which we condition a policy $\pi_z: S \times Z \mapsto \Delta(\mathcal{A})$.
We will treat $p(Z)$ as a categorical distribution over a discrete set $Z$ of $|Z|$ many distinct indicator vectors in $\mathbb{R}^{|Z|}$.
The skill-conditioned policy $\pi_z$ induces a state occupancy denoted by $d_z(s):=d^{\pi_z}(s)$, and when it is clear from the context we will refer to $d_z(s)$ as a ``skill''.

We consider an offline setting with access to the following datasets:
i) $\mathcal{D}_{E}$ sampled from an expert state occupancy $d_{E}(S)$; and 
ii) $\mathcal{D}_{O}$ sampled from a state-action occupancy $d_O(S,A)$ generated by a mixture of behaviors.
Similar to \citet{ma2022smodice}, our analysis makes the following assumption, which requires that the offline state occupancy $d_O$ sufficiently covers the expert's state occupancy $d_E$, a prerequisite for successful imitation learning.
Although this assumption is not required in practice, it ensures well-defined state occupancy measures (i.e., avoiding division by zero).
\begin{asm}[Expert coverage]\label{asm:base}
    We assume that $d_E(s)>0$ implies $d_O(s)>0$.
\end{asm}

\section{Method}\label{sec:Method}

Given an expert and a coverage dataset as above, we aim to solve \textit{offline} the constrained optimization problem 
\begin{eqnarray}
\max_{\{d_z(S)\}_{z\in Z}} & \gI(S;Z)\label{eq:constrained-problem-MI}\\
\text{subject to} & \Dkl \left(d_{z}(S)||d_{E}(S)\right)\leq\epsilon & \forall z,\label{eq:constrained-problem-constraint}
\end{eqnarray}
where $\gI(S;Z)$ denotes the mutual information between states and skills.
The identity $\gI(S;Z)=\mathbb{\mathbb{E}}_{p(z)}\mathrm{KL}(d_{z}(S)||\mathbb{E}_{z^{\prime}}d_{z^{\prime}}(S))$ shows an important geometric perspective that maximizing mutual information is equivalent to finding a set of $|Z|$ skills whose state occupancies $d_z(S)$ correspond as points on a probability simplex such that these points are positioned on the boundary of an ellipsoid and the pairwise distance between each point and the ellipsoid center is maximized~\citep{zahavy21rewardIsEnough, eysenbach22geometry}.

Henceforth, we shall make use of color coding to highlight the {\color{ourdarkblue} diversity} signal in blue and the {\color{ourdarkorange} imitation} signal in orange.
The preceding problem formulation and our algorithmic framework can be easily extended to capture:
i) objectives in \eqref{eq:constrained-problem-MI} that combine conditional mutual information (c.f. DADS in \citep{sharma2019dads}) and information gain (c.f. DISDAIN in~\citep{strouse2021disdain}); and
ii) general $f$-divergence constraints in \eqref{eq:constrained-problem-constraint}, see~\citet{nachum2020reinforcement, ma2022smodice}.
We leave the study of these variants for future work. 

Since maximizing the mutual information is generally intractable, in line with previous work~\citep{eysenbach19diayn} we assume that the latent skills are sampled uniformly at random, i.e., $p(z)=\frac{1}{|Z|}$, and as a trackable surrogate we consider instead the following variational lower bound
\begin{equation}\label{eq:VarLB}
\mathcal{I}\left(S;Z\right)\geq\mathbb{E}_{p(z),d_{z}(s)}\left[\log\cdiv{ q(z|s)}\right]+\mathcal{H}\left(p(z)\right)=\sum_{z}\mathbb{E}_{d_{z}(s)}\left[\frac{\log\left(|Z|\cdiv{q(z|s)}\right)}{|Z|}\right].
\end{equation}
Here with $\cdiv{q(z|s)}$ we denote a skill-discriminator tasked with distinguishing between latent skills.

\citet{ma2022smodice} proposed an offline algorithm (SMODICE) that on input an expert dataset $\mathcal{D}_{E}\sim d_{E}(S)$ and a coverage dataset $\mathcal{D}_{O}\sim d_O(S,A)$ such that $\mathcal{D}_{E}\subset\mathrm{States}[\mathcal{D}_{O}]$, trains a policy $\pi_{\widetilde{E}}$ which optimizes the problem
\begin{equation}\label{eq:smodica-problem}
\min_{\pi}\Dkl\left(d^{\pi}(S)||d_{E}(S)\right),
\end{equation}
and outputs the associated expert ratios $\cdem{\etae(s,a)}=d_{\widetilde{E}}(s,a)/d_O(s,a)$ for every state-action pair $(s,a)\in\mathcal{D}_O$, where $\dE(s,a)$ denotes the state-action occupancy induced by the recovered expert policy $\pi_{\widetilde{E}}$.

An important observation is that given the expert ratios $\cdem{\etae(s,a)}$, the state constraints~\eqref{eq:constrained-problem-constraint} can be relaxed to constraints with respect to the recovered expert state-action occupancy $d_{\widetilde{E}}(s,a)$.
While in theory this relaxation restricts the imitation to the state-action occupancy of a specific expert, it also admits a simpler estimator (see \Cref{lem:opt-lambda}) that is more stable to compute, yields faster runtime performance in practice, and simultaneously provides enough capacity for diversity by increasing the level $\epsilon$.
More specifically, for each latent skill $z$ we replace the state constraint~\eqref{eq:constrained-problem-constraint} with the following state-action constraint
\begin{equation}\label{eq:constrained-kl-state-action}
\Dkl\left(d_{z}(S,A)||\dE(S,A)\right)\leq\epsilon.
\end{equation}

We focus on a reduction of CMDPs to MDPs using gradient-based techniques, known as Lagrangian methods~\citep{borkar2005actor,bhatnagar2012online,tessler2018reward}. 
In contrast to prior work on CMDP, which has focused primarily on linear objectives and constraints, we consider the nonlinear setting with convex objectives and constraints.
More specifically, we seek to maximize the right-hand side of \cref{eq:VarLB} subject to \cref{eq:constrained-kl-state-action}. 
Solving this problem is equivalent to
\begin{equation}\label{eq:lagrange-relaxation}
\max_{{d_{z}(s,a)\atop q(z|s)}}\min_{\lambda\geq0}\sum_{z}\mathbb{E}_{d_{z}(s)}\left[\frac{\log\left(|Z|\cdiv{q(z|s)}\right)}{|Z|}\right]+\sum_{z}\lambda_{z}\left[\epsilon-\Dkl\left(d_{z}(S,A)||\dE(S,A)\right)\right],
\end{equation}
where with $\lambda_z$ we denote the Lagrange multiplier corresponding to latent skill $z$.

\subsection{Approximation Scheme}\label{sec:alt-opt}

\begin{figure*}[htbp]
    \centering
    \includegraphics[width=0.85\linewidth]{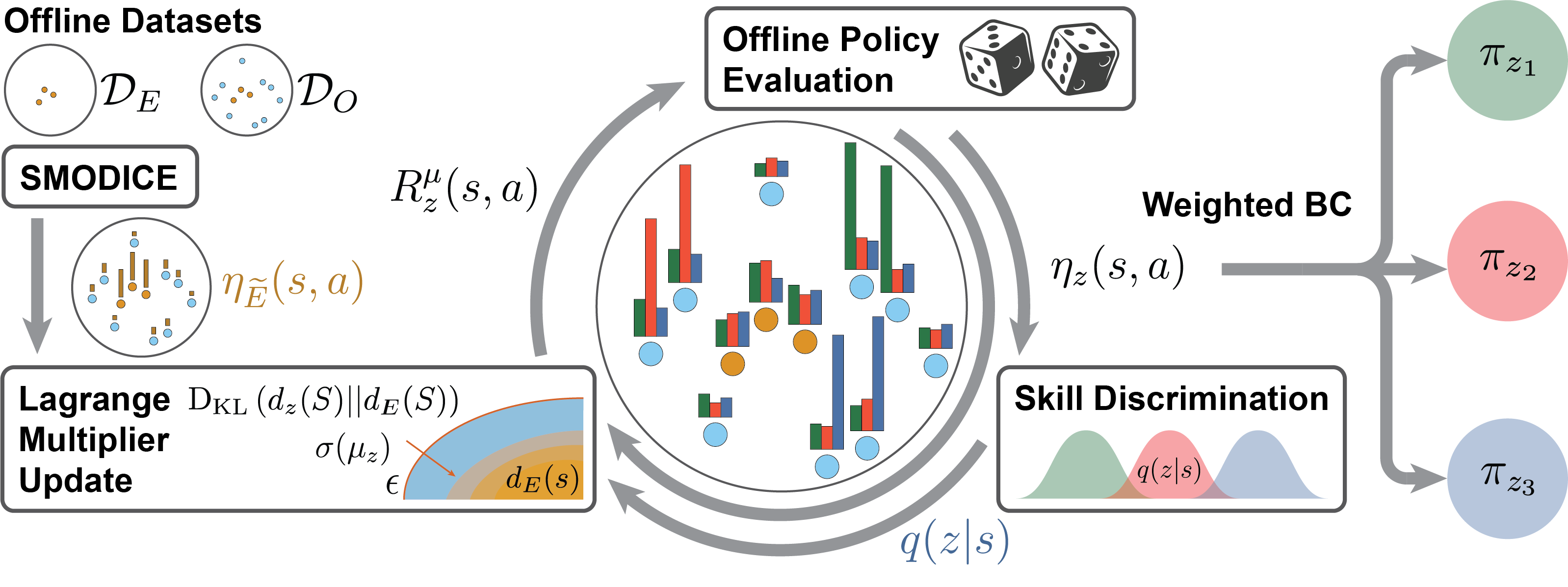}
    \caption{{\small Illustration of \Cref{alg:doi}. 
    We compute expert importance ratios $\cdem{\etae(s,a)}$ by running SMODICE on the offline datasets $\mathcal{D}_{E}$ and $\mathcal{D}_{O}$.
    These expert ratios are then used in the alternating scheme described in \cref{sec:alt-opt} to obtain the importance ratios $\etaz$ (with support in $\mathcal{D}_{O}$) for each skill $z$.
    Specifically, the skill-ratios $\etaz$ are computed by a DICE-like offline policy evaluation algorithm on input a reward $R_z^\mu(s,a)$ that balances skill diversity (skill-discriminator $\cdiv{q(z \vert s)}$) and expert imitation (importance ratios $\cdem{\etae(s,a)}$). 
    }}
    \label{fig:method}
\end{figure*}

We use a popular heuristic, known in the literature as \textit{alternating optimization}, to approximately compute a local optimum of Problem \eqref{eq:lagrange-relaxation}.
More precisely, the method alternates between optimizing each model while holding all others fixed, and iteratively refines the solution until convergence is reached or a stopping criterion is met.
Furthermore, as we can guarantee in practice that the Lagrange multipliers $\lambda$ are always positive, we consider Problem~\eqref{eq:lagrange-relaxation} with $\lambda>0$, that is
\begin{equation}\label{eq:objective}
\max_{{d_{z}(s,a)\atop q(z|s)}}\min_{\lambda>0}\sum_{z}\lambda_{z}\Big\{\epsilon+\mathbb{E}_{d_{z}(s,a)}\left[R_{z}^{\lambda}(s,a)\right]-\Dkl\left(d_{z}(S,A)||d_{O}(S,A)\right)\Big\},
\end{equation}
where
\begin{equation}\label{eqn:reward}
R_{z}^{\lambda}(s,a):=
\underbrace{\frac{1}{\lambda_{z}}}_{\text{Constraint Violation}}\underbrace{\frac{\log\left(\cdiv{q(z|s)}|Z|\right)}{|Z|}}_{\text{Skill Diversity}}+\underbrace{\vphantom{\frac{Z}{\|\lambda_z}}\log\cdem{\etae(s,a)}}_{\text{Expert Imitation}}.
\end{equation}
The reward in~\eqref{eqn:reward} is derived in \cref{app:sec:Lagrange} and relies on the following equality (see \cref{app:sec:SA-KLdiv}) 
$\mathrm{D}_{\mathrm{KL}}\big(d_{z}(S,A)||\dE(S,A)\big)=\mathrm{D}_{\mathrm{KL}}\big(d_{z}(S,A)||d_{O}(S,A)\big)-\mathbb{E}_{d_{z}(s,a)}\big[\log(\dE(s,a)/d_{O}(s,a))\big]$ and the definition of $\cdem{\etae(s,a)}=\dE(s,a)/d_{O}(s,a)$.

Intuitively, the reward $R_{z}^{\lambda}(s,a)$ balances between diversity and KL-closeness to the expert state-action occupancy.
The Lagrange multiplier $\lambda_z$ scales down the log-likelihood of the skill-discriminator $\cdiv{q(z|s)}$, effectively reducing the diversity signal, when the state-action occupancy $d_z(S,A)$ violates the KL-divergence constraint \eqref{eq:constrained-kl-state-action}, and vice versa.
Each term in the reward \eqref{eqn:reward} involves a separate optimization procedure, which will be described in the next section.

\subsection{Approximation Phases}
Using the alternating optimization scheme, \Cref{alg:doi} decomposes into the following three optimization phases.
In \textsc{Phase} 1, we train a value function $V_{z}^{\star}$, ratios $\eta_z(s,a)$ and a skill-conditioned policy $\pi_z$. 
In \textsc{Phase} 2, we train a skill-discriminator $\cdiv{q(z|s)}$.
Then in \textsc{Phase} 3, we compute a KL constraint estimator $\phi_z$ and update accordingly the Lagrange multipliers $\lambda_z$.
In addition, we perform a preprocessing phase to compute the expert ratios $\cdem{\etae(s,a)}$ by invoking the SMODICE algorithm.

\subsubsection{Phase 1}

With fixed skill-discriminator $\cdiv{q(z|s)}$ and Lagrange multipliers
$\lambda>0$, Problem~\eqref{eq:objective} becomes
\begin{equation}
\max\limits _{\substack{\{d_{z}(s,a)\}_{z\in Z}}
}\sum_{z}\lambda_{z}\Big\{\mathbb{E}_{d_{z}(s,a)}\left[R_{z}^{\lambda}(s,a)\right]-\Dkl\left(d_{z}(S,A)||d_{O}(S,A)\right)\Big\},
\end{equation}
or equivalently for every skill $z$:
\begin{eqnarray}\label{eq:LB_fixed_discr_offline}
\max\limits _{\substack{d_{z}(s,a)\geq0}
}&\mathbb{E}_{d_{z}(s,a)}\left[R_{z}^{\lambda}(s,a)\right]-\mathrm{D}_{\mathrm{KL}}\left(d_{z}(S,A)||d_{O}(S,A)\right)&\nonumber\\
\text{subject to} &\sum_{a}d_{z}(s,a)=(1-\gamma)\rho_{0}(s)+\gamma\mathcal{T}d(s)&\forall s,
\end{eqnarray}
where we denote with $\mathcal{T}$ the transition operator: $\mathcal{T}d(s^{\prime})=\sum_{s,a}\mathcal{P}(s^{\prime}|s,a)d(s,a)$.

\begin{asm}[Strict Feasibility]\label{asm:strict_feasibility}
    We assume there exists a solution such that the constraints \eqref{eq:LB_fixed_discr_offline} are satisfied and $d(s,a)>0$ for all states-action pairs $(s,a)\in\mathcal{S}\times\mathcal{A}$.
\end{asm}

Using Lagrange duality, \cref{asm:strict_feasibility} (which implies strong duality) and the Fenchel conjugate (see \cref{app:FenchelConjugate}), \citet[Sec. 6]{nachum2020reinforcement} and \citet[Theorem 2]{ma2022smodice} showed that Problem~\ref{eq:LB_fixed_discr_offline} shares the same optimal value as the following optimization problem
\begin{equation}\label{eq:Vstar}
V^{\star}=\argmin_{V(s)}  (1-\gamma)\mathbb{E}_{s\sim\rho_{0}}\left[V(s)\right]
+ \log\mathbb{E}_{d_O(s,a)}\exp\left\{ R_{z}^{\lambda}(s,a)+\gamma\mathcal{T}V(s,a)-V(s)\right\},
\end{equation}
where $\mathcal{T}V(s,a):=\mathbb{E}_{\mathcal{P}(s^{\prime}|s,a)}V(s^{\prime})$.
Moreover, the primal optimal solution is given by
\begin{equation}\label{eq:dopt}
\etaz:=\frac{d_{z}^{\star}(s,a)}{d_{O}(s,a)}=\mathrm{softmax}_{d_{O}(s,a)}\left(R_{z}^{\lambda}(s,a)+\gamma\mathcal{T}V_{z}^{\star}(s,a)-V_{z}^{\star}(s)\right),
\end{equation}
where $\mathrm{softmax}_{p(x)}(g(x))=\exp\{g(x)\}\Big/\mathbb{E}_{p(x^{\prime})}[\exp\{g(x^{\prime})\}]$.
These ratios $\eta_z(s,a)$ are then used to design an offline importance-weighted sampling procedure that, for an arbitrary function $f$, satisfies
\begin{equation}\label{imp-weight-proc}
\mathbb{E}_{p(z)}\mathbb{E}_{d_{z}^{\star}(s,a)}\big[f(s,a,z)\big]=\mathbb{E}_{p(z)}\mathbb{E}_{d_{O}(s,a)}\big[\eta_{z}(s,a) f(s,a,z)\big].
\end{equation}

Afterwards, the optimal skill-conditioned policy $\pi_z^{\star}$ is trained offline using a weighted behavioral cloning, which is obtained by setting $f(s,a,z)=\log(\pi_z(a|s))$ and maximizing the RHS of \cref{imp-weight-proc} over all skill-conditioned policies $\pi_z$.
In practice, gradient descent is used for optimization.

\subsubsection{Phase 2}
We now give an offline procedure for training a skill-discriminator $\cdiv{q(z | s)}$, which takes as input ratios $\eta_z(s,a)$ of a skill-conditioned policy $\pi^\star_z$.
The proof is presented in \cref{app:sec:train-discriminator}.

\begin{lem}\label{lem:skill-discriminator}
    Given ratios $\eta_{z}(s,a)$, using \cref{imp-weight-proc} applied with $f(s,a,z)=\log( \cdiv{q(z|s)} )$, we can compute offline an optimal skill-discriminator $\cdiv{q^{\star}(z|s)}$. 
    In particular, we optimize by gradient descent the following optimization problem
    $\max_{\cdiv{q(z|s)}}\mathbb{E}_{p(z)}\mathbb{E}_{d_{O}(s,a)}\left[\eta_{z}(s,a) \log\left(\cdiv{q(z|s)}\right)\right]$.
\end{lem}

The skill-conditioned policy $\pi_z^{\star}$ (\textsc{Phase} 1) and the skill-discriminator $\cdiv{q^{\star}}$ (\textsc{Phase} 2), allow us to maximize \textit{offline} the variational lower bound in \cref{eq:VarLB} and thus skill diversity.
It remains to estimate possible constraint violations in \cref{eq:constrained-kl-state-action} and to update the Lagrange multipliers accordingly.

\subsubsection{Phase 3}
With fixed skill-discriminator $\cdiv{q^{\star}(z|s)}$ and skill-conditioned policy $\pi_{z}^{\star}(s)$, Problem~\eqref{eq:objective} reduces to
$\min_{\lambda>0}\sum_{z}\lambda_{z}\big[\epsilon-\mathrm{D}_{\mathrm{KL}}\big(d_{z}^{\star}(S,A)||\dE(S,A)\big)\big]$.
We optimize the Lagrange multipliers by gradient descent.
To this end, we now give an offline estimator of the KL-divergence term.
The proof is presented in \cref{app:sec:SA-KLdiv}.

\begin{lem}\label{lem:opt-lambda}
    Given skill-conditioned policy ratios $\eta_{z}(s,a)$ and expert ratios $\cdem{\etae(s,a)}$, using \cref{imp-weight-proc} applied with $f(s,a,z)=\log( \eta_z(s,a)/\cdem{\etae(s,a)} )$, we can compute offline an estimator of $\mathrm{D}_{\mathrm{KL}}\big(d_{z}^{\star}(S,A)||\dE(S,A)\big)$ which is given by $\phi_{z}:=\mathbb{E}_{d_{O}(s,a)}\big[\eta_{z}(s,a)\log(\eta_{z}(s,a)/\cdem{\etae(s,a)})\big]$.
\end{lem}

We note that the ratios $\eta_{z}(s,a)$ and $\cdem{\etae(s,a)}$ are computed only on state-action pairs within the offline dataset $\mathcal{D}_O$.
Furthermore, in practice, we ensure that these ratios are strictly positive, so that the KL estimator $\phi_z$ is well defined and bounded.


\section{Algorithm}\label{sec:algo}

Our optimization method consists of three phases, each of which optimizes a specific model and fixes the remaining ones.
It is important to emphasize that in contrast to prior work, our problem formulation considers an optimization problem with constraints.
Furthermore, the reward function in \cref{eqn:reward} is non-stationary, since it depends on the bounded Lagrange multipliers that balance diversity ($\log\cdiv{q(z|s)}$) and expert imitation ($\log\cdem{\etae(s,a)}$).
This has significant algorithmic implications, as it requires solving a sequence of standard RL problems, each of which admits offline policy evaluation.

To smooth the transition of the reward signal between successive iterations, we enforce a slow change of the Lagrange multipliers.
More specifically, we use the technique of bounded Lagrange multipliers~\citep{StookeAA20, zahavy2023domino}, which applies a Sigmoid transformation $\lambda=\sigma(\mu)$ component-wise to unbounded variables $\mu\in\mathbb{R}^{|Z|}$, so that the effective reward is a convex combination of a diversity term and an expert imitation term.
In practice, this transformation ensures that $\lambda>0$.
Hence, the reward for each latent skill $z$ becomes
\begin{equation}\label{eq:sigma-reward}
R_{z}^{\mu}(s,a):=\left(1-\sigma(\mu_{z})\right)\frac{\log\left(\cdiv{q^{\star}(z|s)}|Z|\right)}{|Z|}+\sigma(\mu_{z})\log\cdem{\etae(s,a)}.
\end{equation}

We now present the resulting multi-phase optimization procedure in \Cref{alg:doi}.
For the offline training of the policy (in Phase 1), the skill-discriminator (in Phase 2), and the estimation of the KL divergence value (in Phase 3), we use importance sampling~\cref{imp-weight-proc} and give the corresponding empirical estimators in~\cref{app:sec:importance_sampling}.
Our practical implementation leverages the power of neural networks and deep learning techniques for accurate function approximation.
More specifically, we train an expert policy $\pi_{\widetilde{E}}$, a skill-conditioned policy $\{ \pi_z \}_{z\in Z}$ and a value function $\{ V_z \}_{z\in Z}$.
While practically convenient, this means that each phase of \Cref{alg:doi} is only approximately solved.
In particular, we do not solve the optimization problem to optimality in each phase, but rather perform a few gradient descent steps.

\begin{algorithm}[htb]
    {\small
        \caption{Diverse Offline Imitation (DOI)}
        \label{alg:doi}
        \textbf{Input:} a state-only expert dataset $\mathcal{D}_{E}\sim d_{E}(S)$ and an offline dataset $\mathcal{D}_O\sim d_O(S,A)$ such that $\mathcal{D}_{E}\subset\mathrm{States}[\mathcal{D}_{O}]$.
        
        \textbf{Pre-compute} a state-discriminator $c^{\star}:\mathcal{S}\rightarrow(0,1)$ via optimizing the following objective with the gradient penalty in~\citep{gulrajani2017improved}
        $\min_{c}\mathbb{E}_{d_{E}(s)}[\log c(s)]+\mathbb{E}_{d_{O}(s)}[\log(1-c(s))]$

        Apply \textbf{Phase 1} with reward $R(s,a)=\log\frac{c^{\star}(s)}{1-c^{\star}(s)}$ to compute ratios $\cdem{\etae(s,a)}=\dE(s,a)/d_{O}(s,a)$ for all $s,a\in\mathcal{D}_O$
        \vspace{0.3cm}
        
        \textbf{Repeat until convergence:}
        
        $\quad$\textbf{Phase 1.} (Fixed Lagrange multipliers $\sigma(\mu)$ and skill-discriminator values $\cdiv{q^{\star}(z|s)}$)
        
        $\quad$\textbf{For} each latent skill $z$:
        
        $\quad$$\quad$ compute a value function $V_{z}^{\star}$ optimizing \cref{eq:Vstar} with reward $R_{z}^{\mu}(s,a)$ in \cref{eq:sigma-reward}
        
        $\quad$$\quad$ compute ratios $\eta_{z}(s,a)=\mathrm{softmax}_{d_O(s,a)}\left(R_{z}^{\mu}(s,a)+\gamma\mathcal{T}V_{z}^{\star}(s,a)-V_{z}^{\star}(s)\right)$ for all $s,a\in\mathcal{D}_O$

        $\quad$$\quad$ train a skill-conditioned policy $\pi_{z}^{\star}=\arg\max_{\pi_z}\mathbb{E}_{d_{O}(s,a)}\big[\eta_{z}(s,a)\log\pi_{z}(a|s)\big]$
        \vspace{0.3cm}
        
        $\quad$\textbf{Phase 2.} (Fixed ratios $\eta_z(s,a)$ and bounded Lagrange multipliers $\sigma(\mu)$)
        
        $\quad$ Train a skill-discriminator $\cdiv{q^{\star}}=\arg\max_{\cdiv{q(\cdot|s)}}\mathbb{E}_{p(z)}\mathbb{E}_{d_{O}(s,a)}\big[\eta_{z}(s,a)\log \cdiv{q(z|s)}\big]$
        \vspace{0.3cm}
        
        $\quad$\textbf{Phase 3.} (Fixed ratios $\cdem{\etae(s,a)}$ and $\eta_z(s,a)$)
        
        $\quad$Compute for each latent skill $z$ an estimator 
        $\phi_{z}:=\mathbb{E}_{d_{O}(s,a)}\big[\eta_{z}(s,a)\log(\eta_{z}(s,a)/\cdem{\etae(s,a)})\big]$
        
        $\quad$Optimize the loss $\min_{\mu}\sum_{z}\sigma(\mu_{z})(\epsilon-\phi_{z})$
    }
\end{algorithm}

We have found that fitting the skill-discriminator $\cdiv{q(z|s)}$ is prone to collapse to the uniform distribution.
To alleviate this issue, in addition to the variational lower bound objective \eqref{eq:VarLB}, we add the DISDAIN information gain term, proposed in~\citep{strouse2021disdain}.
This bonus term is an entropy-based disagreement penalty that estimates the epistemic uncertainty of the skill-discriminator, and is implemented in practice by an ensemble of randomly initialized skill-discriminators.
Due to the high initial disagreement on unvisited states, this intrinsic reward provides a strong exploration signal and leads to the discovery of more diverse behaviors.
Intuitively, for states with small epistemic uncertainty, the skill-discriminator (averaged over the ensemble members) should reliably discriminate between latent skills, thus making the intrinsic reward of the skill-discriminator's log-likelihood more accurate.

\section{Experiments}
\label{sec:experiments}

For evaluation of our method we consider 12 degree-of-freedom quadruped robot \envsolo{}~\citep{grimminger2020open}, on a simple locomotion task in both the \emph{simulation} and the \emph{real} system.
We complement this with an obstacle navigation task, in simulation, and demonstrate that some of the learned diverse skills robustly reach a target position while the expert fails.
Furthermore, we provide evaluation on the \envant{}, \envwalker{}, \envhalfcheetah{} and 
\envhopper{} environments from the standard D4RL benchmark~\citep{fu2020d4rl}.

\subsection{Locomotion}
\label{subsec:locomotion}

\paragraph{Data collection.}
For the \envsolo{} evaluation, we collected domain-randomized offline and expert data from simulation in the Isaac Gym~\citep{makoviychuk2021isaac}, using pretrained policy checkpoints obtained by training the robot to track a certain speed of the base with the on-policy diversity maximization algorithm DOMiNiC~\citep{cheng2024dominic}.
We defer the data collection procedure to the \cref{app:sec:dataset-collection}. 
The \emph{expert dataset} was collected by using the best deterministic policy from the last checkpoint of the training procedure, which was trained to track forward velocity only without diversity objective.
In contrast, the \emph{offline dataset} was acquired by employing stochastic policies gathered from various checkpoints throughout the training of the expert, featuring multiple latent skills.
More than half of the \emph{offline dataset} was collected by a random Gaussian policy.
In line with previous approaches by \citet{kim2022demodice} and \citet{ma2022smodice}, our practical implementation aims to fulfill the expert coverage~\cref{asm:base}.
To achieve this, we create the \emph{coverage dataset} $\mathcal{D}_{O}$ by adding a small number of expert trajectories to the offline dataset, resulting in an \textit{unlabeled} expert fraction of 1/160 in $\mathcal{D}_{O}$.
We discard expert actions from the expert dataset to ensure that our algorithm does not have labeled access to them.
The resulting \emph{expert dataset} $\mathcal{D}_{E}$ is used to learn a state classifier $c(s)$.
Then the SMODICE is executed to compute the importance ratios $\cdem{\etae(s,a)}$, see~\cref{sec:Method}.
We trained the policy for $350$ steps, where each step involves the stages described in \cref{sec:algo}.
In each stage, we execute $200$ epochs of batched training over the data.

\begin{figure}[thbp!]
    \centering
    \begin{subfigure}[b]{0.45\textwidth}
        \centering
        \legendskills{}
        \includegraphics[width=\textwidth]{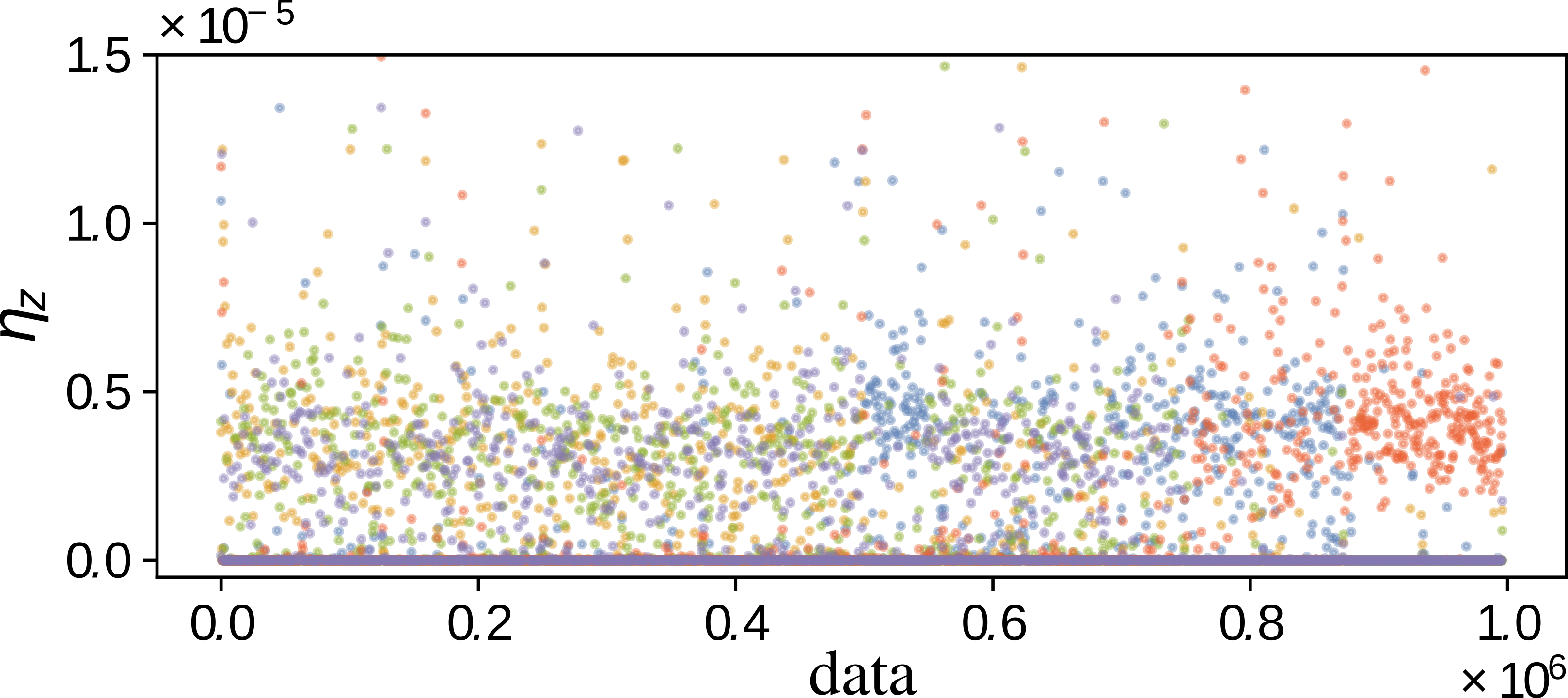}\\
        \vspace{.2em}
        \includegraphics[width=\textwidth]{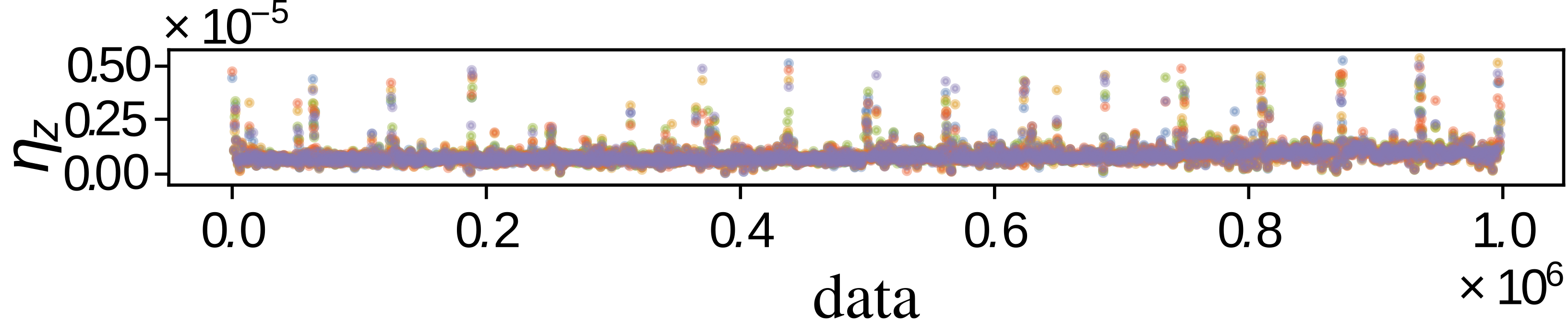}
        \caption{}
        \label{fig:eta_z}
    \end{subfigure}
    \hspace{1em}
    \begin{subfigure}[b]{0.45\textwidth}
        \centering
        \legendtwo{}
        \vspace{1em}
        \includegraphics[width=\textwidth]{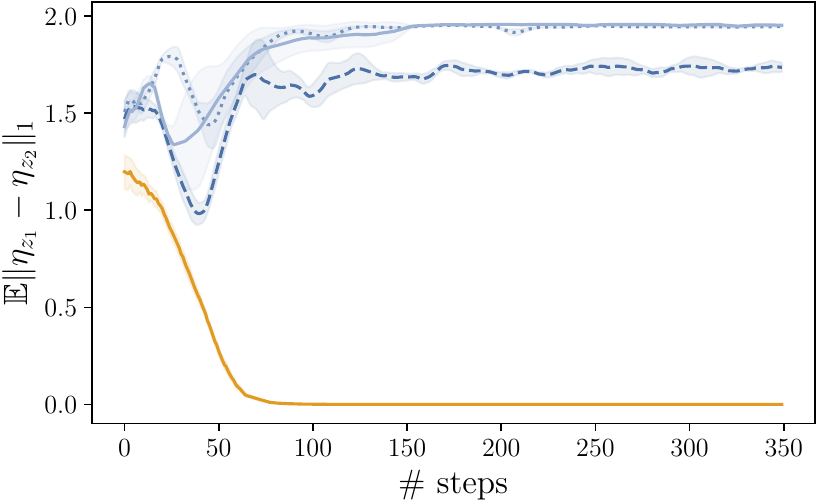}
        \caption{}
        \label{fig:l1_distance_learned}
    \end{subfigure}
    \caption{{\small Data points separation by importance ratios $\eta_z(s,a)$, given different levels of $\epsilon$ in \envsolo{}. (a) Distribution of importance ratios $\eta_z(s,a)$ over the offline dataset $\mathcal{D}_{O}$ for distinct skills with \method{}$^4$ ($\epsilon=4$) (upper) and a skill-conditioned variant of SMODICE (lower). 
    (b) Average $\ell_1$ distance of ratios $\eta_z$ belonging to distinct skills, depending on $\epsilon$. 
    The higher the value of $\epsilon$, the greater the $\ell_1$ distance.
    The shaded areas show the interval between the 0.25 and 0.75 quantiles, computed over \numSeeds seeds.}}
    \label{fig:learned-lambda-res}
\end{figure}

Here with DOI$^\epsilon$ we denote an execution of \Cref{alg:doi} with constraint threshold set to $\epsilon$.
We proceed by analyzing the learned DOI skills in three evaluation settings:
i) over the fixed offline datasets; ii) a Monte Carlo on-policy evaluation in the simulator; and iii) the resulting clustering structure involving the offline and expert datasets, as well as the DOI skills and the SMODICE expert evaluated in the simulation.

\paragraph{Importance ratios distance.}
In \cref{fig:learned-lambda-res}, we measure the state-action occupancy $d_z(s,a)$ for each latent skill $z$ through the proxy of importance ratios $\eta_z(s,a)$,\footnote{
For the computation of the skill-ratios $\eta_z(s,a)$, we choose a projection $\Pi$ of the expert state (see \cref{app:sec:observation-projection}) that yields 3-dimensional planar and angular velocities of the robot's base in the base frame.
} for different values of $\epsilon$.
As expected, a higher value of $\epsilon$ increases diversity, resulting in different importance ratios per skill for individual data points.
This difference is then aggregated by computing an expected $\ell_1$ distance between importance ratios of distinct skills, i.e., $\E \|\eta_{z_i} - \eta_{z_j}\|_1$, and is reported in \cref{fig:learned-lambda-res}.
We note that the looser the constraint (lighter color), the easier it is to diversify in the sense of $\eta_z$.
\cref{fig:l1_distance_learned} shows the average $\ell_1$ distance between skill importance vectors $\eta_z$ over the dataset for $\epsilon \in \{0.0, 1.0, 2.0, 4.0\}$ (lighter color indicates higher $\epsilon$).
Moreover, the tighter the constraint (smaller $\epsilon$), the smaller the difference between the distinct skill importance ratios.

To analyze the influence of the diversity objective on the learned skills, we consider as a baseline a skill-conditioned variant of \citep{ma2022smodice}, denoted SMODICE$^\dagger$, which does not have access to the skill discriminator $\cdiv{q(z|s)}$.
This is equivalent to \method{} with fixed $\sigma(\mu_z)=1$ in the reward \cref{eq:sigma-reward}.
We defer further experiments with fixed Lagrange multipliers to \cref{app:sec:additional-experiments}.
In \cref{fig:eta_z}, we observe diversification across the dataset assignment to skills when using \method{}, whereas training an ensemble of skills with only expert imitation reward (i.e., $\sigma(\mu_z)=1$) collapses to nearly the same importance ratios per skill per data point.

\paragraph{Successor features distance.}
We have further evaluated diversity on the Monte Carlo estimates of the expected successor features over the initial state, based on 30 policy rollouts per skill.
The $\gamma$-discounted successor features (SFs) for state $s$ are defined as $\psi_z(s) = \E_{d_z(s)}[\phi(s)]$, where $d_z(s)$ is the $\gamma$-discounted state occupancy for a skill policy $\pi_z$.
With slight abuse of notation, we define $\psi_z = \E_{\rho_0(s)}[\psi_z(s)]$, the expected SFs over the initial state distribution.
As a diversity metric, we take the expected $\ell_2$ distance between the SFs of distinct skills, i.e., $\E\|\psi_{z_1} - \psi_{z_2}\|_2$.
The results are presented in \cref{fig:div-and-task} and are consistent with the proxy diversity metric.
In particular, there is a correspondence between the offline data separation induced by the importance ratios $\eta_z$ (see~\cref{fig:eta_z}), and a higher distance between the expected SFs $\psi_z$ (see~\cref{fig:task-reward}).
In terms of performance, \method{} achieves a forward velocity comparable to the expert (see~\cref{fig:task-reward}) while learning diverse skills with respect to base height $h$ (see~\cref{fig:height-std}).
We also observed that the multipliers $\sigma(\mu_z)$ are non-zero for all skills, indicating that the constraint is active.
In addition, they stabilize at reasonable levels as training progresses, which we show in \cref{app:sec:lagrange-stability} for both the \envsolo{} and \envant{}.

\begin{figure}[h]
    \centering
    \legend{}
    \vspace{1em}
    \begin{subfigure}[b]{0.45\textwidth}
        \centering
        \includegraphics[width=\textwidth]{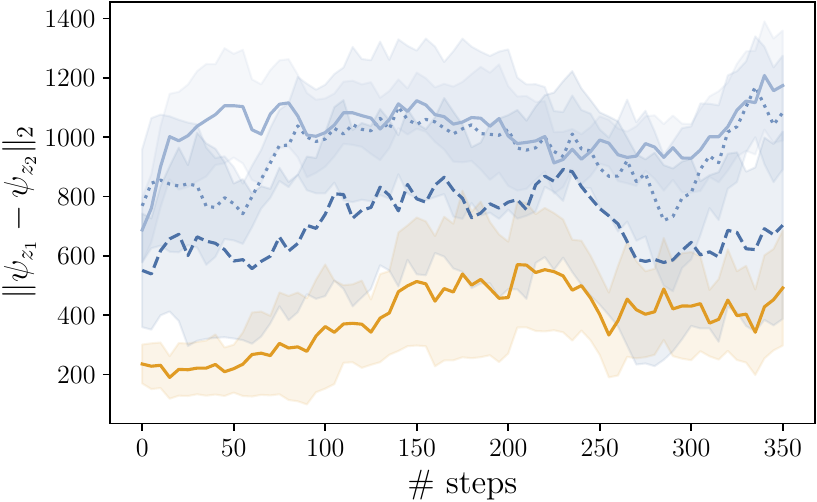}\\
        \caption{}
        \label{fig:task-reward}
    \end{subfigure}
    \hspace{1em}
    \begin{subfigure}[b]{0.45\textwidth}
        \centering
        \includegraphics[width=\textwidth]{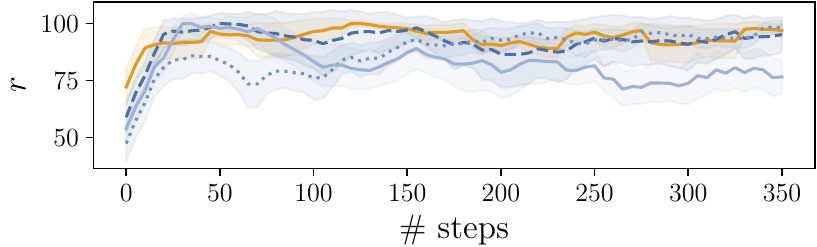}\\
        \includegraphics[width=\textwidth]{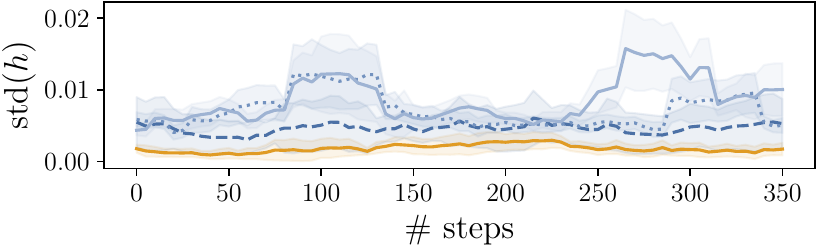}
        \caption{}
        \label{fig:height-std}
    \end{subfigure}
    \caption{{\small (a) Average $\ell_2$ distance between Monte Carlo estimates of successor features $\psi_z$ of distinct skills; (b) return $r$ as $\%$ of expert return and standard deviation of base height $\mathrm{std}_z(h)$. Both depend on $\epsilon$ for the \envsolo{}.
    The shaded areas show the interval between the 0.25 and 0.75 quantiles, computed over \numSeeds seeds.}}
    \label{fig:div-and-task}
\end{figure}

\paragraph{DOI skills form well-separated clusters.}

Here we conduct a controlled experiment with full trajectory information, which remains hidden to the DOI algorithm.
In \cref{fig:well-separated-clusters}, the Successor Features of each trajectory in the {\color{ourdarkbrown} expert dataset} are transformed by UMAP~\citep{McInnes2018umap} algorithm into 2D space.
This transformation is then used to map the SFs of each trajectory into 2D space for: i) the {\color{ourgrey} offline dataset}, ii) the {\color{ourorange} SMODICE expert} evaluated in simulation, and iii) the learned DOI skills ({\color{ourred} red}, {\color{ourgreen} green}, {\color{ourblue} blue}, {\color{ourpurple} purple}, {\color{cyan} cyan}) also evaluated in simulation.
The diversity of learned DOI skills is reflected in a well-separated cluster structure.

\begin{figure}[h]
    \centering
    \includegraphics[scale=0.7]{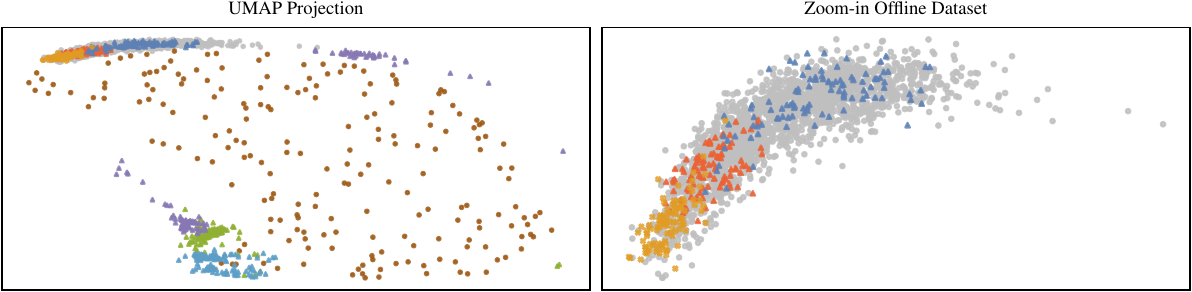}
    \caption{Successor Features projection onto 2D space using the UMAP algorithm.}
    \label{fig:well-separated-clusters}
\vspace{-0.2cm}
\end{figure}

\subsection{Robust Obstacle Navigation}

\paragraph{Data collection.} 
Similarly to the locomotion task in \cref{subsec:locomotion}, both \emph{expert dataset} and \emph{offline dataset} were generated from pretrained policy checkpoints from training a robot to navigate in the terrain of obstacles with fixed time limit using the DOMiNiC~\citep{cheng2024dominic} algorithm. 
Unlike the previous task, the \emph{expert dataset} was collected using the best deterministic skill-conditioned policy from the last checkpoint of the training procedure, which exhibits diverse strategies to navigate the obstacle terrain, including bypassing it from both sides or climbing over it. 
The \emph{offline dataset} was acquired through rolling out stochastic policies gathered from multiple checkpoints with multiple skills. 
Both \emph{expert dataset} and \emph{offline dataset} were collected in a terrain of a single obstacle of a fixed height of 0.2 meters. 
Similar to \cref{subsec:locomotion}, we create the \emph{coverage dataset} $\mathcal{D}_{O}$ by adding a small number of expert trajectories to the offline dataset.
For details on collecting the dataset for the obstacle navigation task, we refer interested readers to the \cref{app:sec:dataset-collection}. 

\paragraph{Multi-modal expert limitations.}
Deriving a single policy by {\color{darkorange} SMODICE} from expert demonstrations, even in the setting when the dataset was collected from diverse expert strategies, may lack robustness to distribution shifts.
This observation emphasizes the need for diverse policy extraction.
To illustrate this with a concrete example, consider a scenario where a \envsolo{} robot navigates around a single box obstacle to reach a target position behind it.
The target position can be reached either from the sides (left or right) of the box or by climbing over it (the less robust path).
In our experiments, the expert dataset $\gD_E$ contains all of the above strategies.
As shown in \Cref{fig:robustness-experiment}, for boxes with a height of at least 0.3 meters, the {\color{darkorange} SMODICE expert} consistently positions itself in front of the box and thus fails to robustly reach the target position.

\paragraph{Extracting robust policies.}
In \Cref{fig:robustness-experiment}, we analyze return distributions and sampled trajectories for box heights of $\{0.3,0.6\}$ meters.
The {\color{darkorange} SMODICE expert} predominantly fails to reach the target position, due to a bias towards climbing over the box.
In contrast, a {\color{ourblue} DOI} skill consistently chooses the left side, which leads robustly to the target position and achieves a superior return distribution.
However, it is important to note that not all learned DOI skills are robust.
Hence, a subsequent selection process is required.
Further details about all learned DOI skills, their return distributions and sampled trajectories, different box heights, and additional experimental information are presented in \cref{app:robustness}.



\begin{figure}[h]
\newlength{\gridspace}
\setlength{\gridspace}{1.2em}
\newcommand\scale{0.8}
\begin{center}
\begin{tabular}{c c c c}
    \hspace{\gridspace}\includegraphics[scale=\scale]{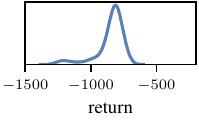} &
    
    \hspace{\gridspace}\includegraphics[scale=\scale]{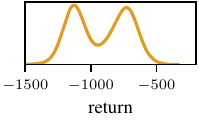} &
    
    \hspace{\gridspace}\includegraphics[scale=\scale]{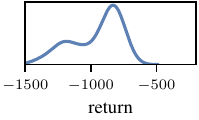} &
    
    \hspace{\gridspace}\includegraphics[scale=\scale]{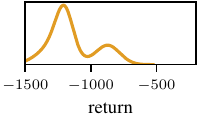} \\
    
    \includegraphics[scale=\scale]{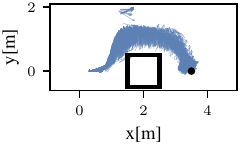} & 
    
    \includegraphics[scale=\scale]{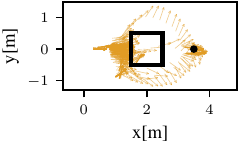} &
    
    \includegraphics[scale=\scale]{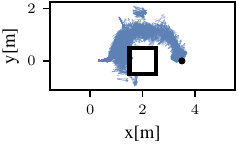} &
    
    \includegraphics[scale=0.85]{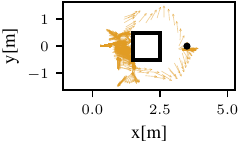} \\
    \multicolumn{2}{c}{\hspace{\gridspace}{\small (a)}} & \multicolumn{2}{c}{\hspace{\gridspace}{\small (b)}} 
\end{tabular}
\end{center}
\caption{{\small Return distributions and sampled trajectories of {\color{darkorange} SMODICE} and a {\color{ourblue} DOI} skill for terrains with box height (a) $0.3$ and (b) $0.6$.
The heights of the boxes are out-of-distribution for the {\color{darkorange} SMODICE}, which tends to get stuck in front of the box due to a bias towards climbing over it.
In contrast, the robust {\color{ourblue} DOI} skill takes a detour to the left side of the box.}}
\label{fig:robustness-experiment}
\end{figure}


\subsection{Standard D4RL Environments}
We consider the case where we have offline data generated from a random policy mixed with a small amount of expert trajectories.\footnote{The same setting was considered by \citet{ma2022smodice}.}
\Cref{fig:d4rl-results} shows the results for both the expected $\ell_2$ distance between the SFs or the importance ratios $\eta_z$ of distinct skills.
We normalize the state feature $\phi(s)$ when comparing SFs $\psi_z$ across environments in \cref{fig:d4rl-diversity}.
In most cases, we report a trade-off between the average skill return and the imitation level $\epsilon$.
The larger the imitation slack $\epsilon$, the more diverse the skills become, but at the cost of lowering the average return, and vice versa.
Nevertheless, in \cref{fig:d4rl-diversity} we show that $\epsilon$ retains some controllability over diversity.
The \envwalker{} is particularly sensitive to relaxation of the occupancy constraint with respect to performance.
We hypothesize that this is due to the fact that the space of policies that achieve a stable gait is very restrictive, resulting in a significant loss of task return for even a small amount of skill diversity.
In contrast, the \envant{} exhibits high stability, with several skills achieving close to expert performance in terms of $r$.
These results are also consistent with {\color{darkorange} SMODICE expert} policies used for computing $\cdem{\etae(s,a)}$ (see \cref{app:sec:smodice-results}).

\begin{figure}[h]
    \centering
    \legendenvs{}
    
    
    \vspace{1em}
    \begin{subfigure}[b]{0.4\textwidth}
        \centering
        \includegraphics[width=\textwidth]{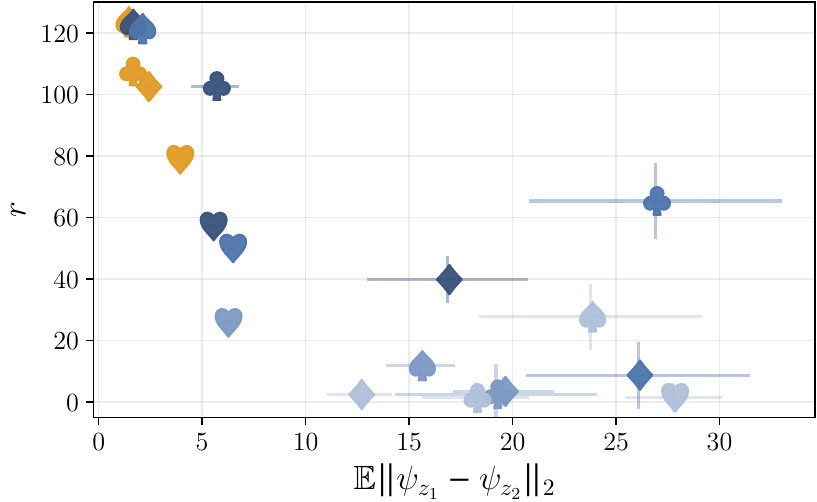}\\
        \caption{}
        \label{fig:d4rl-diversity}
    \end{subfigure}
    \hspace{1em}
        \begin{subfigure}[b]{0.4\textwidth}
        \centering
        \includegraphics[width=\textwidth]{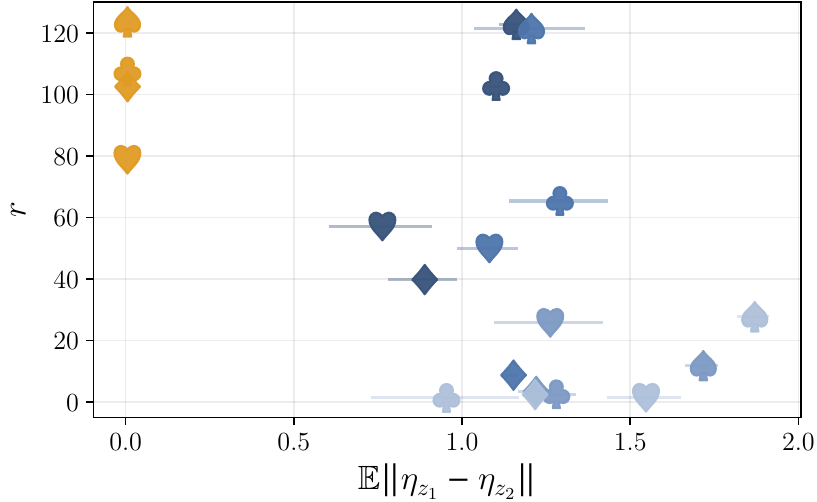}\\
        \caption{}
        \label{fig:d4rl-w_e_dist}
    \end{subfigure}
    \caption{{\small  Results on D4RL environments with offline data collected from a random policy for $\epsilon=$ 
 {\color{ourorange}0.0}, {\color{sc4}0.5}, {\color{sc3}1.0}, {\color{sc2}2.0}, {\color{sc1}4.0}. In figure (a) we observe the tradeoff between average skill return and average successor features distance over skills. 
 In figure (b), we report the tradeoff w.r.t. average $\ell_1$ distance of importance ratios $\eta_z$.}
 The lines indicate the standard deviation computed over \numSeeds seeds.}
    \label{fig:d4rl-results}
\end{figure}

\section{Conclusion}\label{sec:conclusion}

We proposed \method{}, a principled offline RL algorithm for unsupervised skill discovery that, in addition to maximizing diversity, ensures that each learned skill imitates state-only expert demonstrations to a certain degree.
Our main analytical contribution is to connect Fenchel duality, reinforcement learning, and unsupervised skill discovery to maximize a mutual information objective subject to KL-divergence state occupancy constraints.
We have shown that \method{} can diversify offline policies for a 12-DoF quadruped robot (in simulation and in reality) and for several environments from the standard D4RL benchmark in terms of both $\ell_2$ distance of expected successor features and $\ell_1$ distance of importance ratios, which is visible from the data separation induced by $\eta_z(s,a)$ among skills.
The importance ratio distance, computed offline, is a robust indicator of diversity, which aligns with the online Monte Carlo diversity metric of expected successor features. 
The resulting skill diversity naturally entails a trade-off in task performance.
We can control the amount of diversity via an imitation level $\epsilon$, which ensures that distinct skills remain close to the expert in terms of state-action occupancy, which also indirectly controls task performance loss.
A promising direction for future research is to impose constraints on the value function of each skill to ensure near-optimal task performance.

\section*{Limitations}

Our approach, while promising, is not without limitations. 
The diversity objective, which is given by a variational lower bound on the mutual information between states and skills, necessitates the training of a skill-discriminator. 
This design choice, however, presents several practical challenges:
i) the single-step policy and skill-discriminator update in the offline setting does not provide as accurate a policy estimate as sampling a Monte Carlo trajectory in the online setting \citep{eysenbach19diayn};
ii) this inaccuracy, when combined with the non-stationary reward (bounded Lagrange multipliers and skill-discriminator), could result in a skill-discriminator that fails to accurately discriminate among skills; and 
iii) while the introduction of an additional information gain term, as in \citet{strouse2021disdain}, can alleviate this issue, its effect could fade quickly and serve only as an initial diversity boost in the offline setting.
Furthermore, the current paradigm is well-suited for a discrete number of skills, leaving open the important questions of extending our framework to infinitely many skills and addressing in a principled way the practical challenges surrounding the skill-discriminator training. We leave these important open questions for future work.

\section*{Acknowledgments}

We acknowledge the support from the German Federal Ministry of Education and Research (BMBF) through the Tübingen AI Center (FKZ: 01IS18039B). 
Georg Martius is a member of the Machine Learning Cluster of Excellence, funded by the Deutsche Forschungsgemeinschaft (DFG, German Research Foundation) under Germany’s Excellence Strategy – EXC number 2064/1 – Project number 390727645. This work was supported by the ERC - 101045454 REAL-RL. Pavel Kolev was supported by the Cyber Valley Research Fund and the Volkswagen Stiftung (No 98 571). 
We would like to thank Nico Gürtler for helpful comments on an earlier version of this paper.

\newpage

\bibliography{references}
\bibliographystyle{rlc}

\newpage

\onecolumn

\appendix
{
    \centering
    \begin{center}
        \LARGE
        \textbf{Supplementary for Diverse Offline Imitation Learning}
    \end{center}
}

\renewcommand{\thetable}{S\arabic{table}}
\renewcommand{\thefigure}{S\arabic{figure}}
\renewcommand{\theequation}{S\arabic{equation}}
\setcounter{table}{0}
\setcounter{figure}{0}
\setcounter{equation}{0}

\section{Reproducibility}\label{sec:reproducibility}

For implementation of \method{} we have used the PyTorch autograd  framework.
For the \envsolo{} training we made use of Isaac Gym for data collection and evaluation of the learned skill policies.
For the D4RL experiments we evaluated the policies using the Mujoco v2.1 rigid body simulator.
The training of the skill policies with evaluation and pre-training of the SMODICE expert ratios takes about 4 hours on an NVIDIA GeForce RTX 4080 graphics card with a batch size of 512.
We plan on opensourcing the code and the \envsolo{} data post conference acceptance.
The \envsolo{} robot has been developed as part of the Open Dynamic Robot Initiative~\citep{grimminger2020open}, and a full assembly kit is available at a cheap price in order to reproduce the real system experiments from \cref{app:sec:real-robot-experiments}.

\section{Fenchel Conjugate}\label{app:FenchelConjugate}

The Fenchel conjugate $f_{\star}$ of a function $f:\Omega\rightarrow\mathbb{R}$ is given by 
\begin{equation}
f_{\star}(y)=\sup_{x\in\Omega}\langle x,y\rangle-f(x),    
\end{equation}
where $\langle\cdot,\cdot\rangle$ denotes the inner product defined on a space $\Omega$.
For any proper, convex and lower semi-continuous function $f$ the following duality statement holds $f_{\star\star}=f$, that is 
\begin{equation}
f(x)=\sup_{y\in\Omega_{\star}}\langle x,y\rangle-f_{\star}(y),
\end{equation}
where $\Omega_{\star}$ denotes the domain of $f_{\star}$.
For any probability distributions $p,q\in\Delta(S)$ with $p(s)>0$ implying $q(s)>0$, we define for convex continuous functions $f$ the family of $f$-divergences
\begin{equation}
\Df(p||q)=\E_{q}\left[f\left(\frac{p(x)}{q(x)}\right)\right].
\end{equation}
The Fenchel conjugate of an $f$ divergence $\Df(p||q)$ at a function $y(s) =p(s)/q(s)$ is, under certain conditions\footnote{$f$ needs to satisfy certain regularity conditions~\citep{dai2017learning}}, given by
\begin{equation}
\mathrm{D}_{\star, f}(y) = \E_{q(s)} \left [ f_\star(y(s)) \right ].
\end{equation}
Furthermore, its maximizer satisfies 
\begin{equation}
    p^\star(s) = q(s) f_\star'(y(s)).
\end{equation}

In the important special case where $f(x) = x \log(x)$, we obtain the well-known Kullback-Leibler (KL) divergence 
\begin{equation}
    \Dkl(p||q)=\sum_{s}p(s)\log\frac{p(s)}{q(s)}.
\end{equation}
The Fenchel conjugate $\mathrm{D}_{\star,\mathrm{KL}}$ of the KL-divergence at a function $y(s)=p(s)/q(s)$ has a closed-form~\citep[Example 3.25]{boyd2004convex}
\begin{equation}\label{eq:fencel_conjugate_KL}
\mathrm{D}_{\star,\mathrm{KL}}(y)=\log\mathbb{E}_{q(s)}[\exp{y(s)}],
\end{equation}
and its maximizer $p^{\star}$ satisfies
\begin{equation}\label{eq:conjugate-optimal}
p^\star(s) = q(s) \mathrm{softmax}_{q}(y(s)), \quad \text{where} \quad \mathrm{softmax}_{q}(y(s))=\frac{\exp{y(s)}}{\mathbb{E}_{q(s^{\prime})}[\exp{y(s^{\prime}))}]}
\end{equation}

\section{Lagrange Relaxation}\label{app:sec:Lagrange}
The Lagrange relaxation is given by
\[
\max_{d_{z}(s,a),q(z|s)}\min_{\lambda>0}\sum_{z}\mathbb{E}_{d_{z}(s)}\left[\frac{\log\left(|Z|q(z|s)\right)}{|Z|}\right]+\sum_{z}\lambda_{z}\left[\epsilon-\Dkl\left(d_{z}(S,A)||\dE(S,A)\right)\right].
\]
By combining \cref{app:lem:dz_dEz} and the definition of $\etae(s,a)=\dE(s,a)/d_{O}(s,a)$, we have
\[
\Dkl\left(d_{z}(S,A)||\dE(S,A)\right)=\Dkl\left(d_{z}(S,A)||d_{O}(S,A)\right)-\mathbb{E}_{d_{z}(s,a)}\left[\log\etae(s,a)\right]
\]
and thus
\begin{equation}\label{app:main:problem}
\max_{d_{z}(s,a),q(z|s)}\min_{\lambda>0}\sum_{z}\lambda_{z}\left[\epsilon+\mathbb{E}_{d_{z}(s,a)}\left[R_{z}^{\lambda}(s,a)\right]-\Dkl\left(d_{z}(S,A)||d_{O}(S,A)\right)\right],
\end{equation}
where the reward is given by
\[
R_{z}^{\lambda}(s,a):=\frac{\log\left(|Z|q(z|s)\right)}{\lambda_{z}|Z|}+\log\etae(s,a).
\]

\section{Algorithmic Phases}

\subsection{Value Function Training}\label{app:lem:value-func-training}

With fixed skill-discriminator $q(z|s)$ and Lagrange multipliers $\lambda>0$, the Problem~\ref{app:main:problem} becomes:
\[
\max_{\{d_{z}(s,a)\}_{z\in Z}}\sum_{z}\lambda_{z}\left\{ \mathbb{E}_{d_{z}(s,a)}\left[R_{z}^{\lambda}(s,a)\right]-\mathrm{D}_{\mathrm{KL}}\left(d_{z}(s,a)||d_{O}(s,a)\right)\right\} 
\]
or equivalently for every skill $z$:
\begin{equation}
\begin{array}{rl}
\max\limits_{\substack{d_z(s,a) \geq 0}} & \mathbb{E}_{d_{z}(s,a)}\left[R_{z}^{\lambda}(s,a)\right]-\mathrm{D}_{\mathrm{KL}}\left(d_{z}(S,A)||d_{O}(S,A)\right)\\
\text{s.t.}&\sum_{a}d_{z}(s,a)=(1-\gamma)\rho_{0}(s) +\gamma\mathcal{T}d(s) \quad \forall s.
\end{array}\label{app:LB_fixed_discr_offline}
\end{equation}
We note that the preceding problem formulation involves state-action occupancy.\\

The strict feasibility in Assumption~\ref{asm:strict_feasibility} implies strong duality, and thus Problem~\eqref{app:LB_fixed_discr_offline} shares the same optimal value as the following dual minimization problem (for details see \citep[Section 6]{nachum2020reinforcement} and \citep[Theorem 2]{ma2022smodice}):
\begin{equation}
\begin{array}{rl}
V^{\star}= & \argmin_{V(s)} (1-\gamma)\mathbb{E}_{s\sim\rho_{0}}\left[V(s)\right]\\
& +\log\mathbb{E}_{d_{O}(s,a)}\exp\left\{ R_{z}^{\lambda}(s,a)+\gamma\mathcal{T}V(s,a)-V(s)\right\},
\end{array}
\end{equation}
where
\[
\mathcal{T}V(s,a)=\mathbb{E}_{\mathcal{P}(s^{\prime}|s,a)}V(s^{\prime}).
\]
Moreover, the optimal primal solution reads 
\begin{equation}
\frac{d_{z}^{\star}(s,a)}{d_{O}(s,a)}=\mathrm{softmax}_{d_{O}(s,a)}\left(R_{z}^{\lambda}(s,a)+\gamma\mathcal{T}V_{z}^{\star}(s,a)-V_{z}^{\star}(s)\right).
\end{equation}

\subsection{Skill Discriminator Training}\label{app:sec:train-discriminator}

With fixed skill-conditioned policy $\pi_{z}^{\star}$ and Lagrange multipliers $\lambda>0$, the Problem~\ref{app:main:problem} becomes
\begin{eqnarray*}
    \max_{q(z|s)}\sum_{z}\left\{ \mathbb{E}_{d_{z}(s,a)}\left[R_{z}^{\lambda}(s,a)\right]-\Dkl\left(d_{z}(S,A)||d_{O}(S,A)\right)\right\}
\end{eqnarray*}
and reduces to
\[\label{eq:discriminator-training}
\max_{q(z|s)}\mathbb{E}_{p(z)}\mathbb{E}_{d_{z}(s,a)}\log q(z|s).
\]

\begin{lem}
    Given ratios $\eta_{z}(s,a)$, using weighted-importance sampling, we can train offline an optimal skill-discriminator $q(z|s)$. 
    In particular, we optimize by gradient descent the following optimization problem
    \[
    \max_{q(z|s)}\mathbb{E}_{p(z)}\mathbb{E}_{d_{O}(s,a)}\left[\eta_{z}(s,a)\log q(z|s)\right].
    \]
\end{lem}
\begin{proof}
    The statement follows by combining \cref{eq:discr_grad} and \cref{app:lem:IS}.
\end{proof}

\begin{lem}[Discriminator Gradient]\label{eq:discr_grad}
    It holds that
    \[
    \nabla_{\phi}\mathbb{E}_{p(s)}\left[\Dkl\left(p(Z|s)||q_{\phi}(Z|s)\right)\right]=-\mathbb{E}_{p(z)}\mathbb{E}_{p(s|z)}\left[\nabla_{\phi}\log q_{\phi}(z|s)\right].
    \] 
\end{lem}
\begin{proof}
    Observe that
    \begin{eqnarray*}
        \nabla_{\phi}\Dkl\left(p(Z|s)||q(Z|s)\right)&=&\nabla_{\phi}\mathbb{E}_{p(z|s)}\log\frac{p(z|s)}{q_{\phi}(z|s)}\\&=&-\mathbb{E}_{p(z|s)}\nabla_{\phi}\log q_{\phi}(z|s),
    \end{eqnarray*}
    where the second equality follows by
    \begin{eqnarray*}
        \nabla_{\phi}\log\frac{p(z|s)}{q_{\phi}(z|s)}=-\frac{q_{\phi}(z|s)}{p(z|s)}p(z|s)\frac{\nabla_{\phi}q_{\phi}(z|s)}{[q_{\phi}(z|s)]^{2}}=-\frac{\nabla_{\phi}q_{\phi}(z|s)}{q_{\phi}(z|s)}=-\nabla_{\phi}\log q_{\phi}(z|s).    
    \end{eqnarray*}
\end{proof}

\subsection{KL-divergence Constraint Violation}\label{app:sec:SA-KLdiv}

\begin{lem}[State-Action KL Estimator]\label{app:lem:SA-KL-Estimator}
    Suppose we are given offline datasets $\mathcal{D}_{O}(S,A)\sim d_O$, $\mathcal{D}_{E}(S) \sim d_E$ and optimal ratios $\eta_{z}(s,a)=\frac{d_{z}(s,a)}{d_{O}(s,a)}$ and $\etae(s,a)=\frac{\dE(s,a)}{d_{O}(s,a)}$ for all $(s,a)\in\mathcal{D}_O$, where the state-action occupancy $\dE$ is induced by a policy $\pi_{\widetilde{E}}$ agreeing on the state occupancy of an expert $\pi_{E}$, i.e.
    \[
    \pi_{\widetilde{E}}\in\arg\min_{\pi}\mathrm{D}_{\mathrm{KL}}\left(d_{\pi}(S)||d_{E}(S)\right).
    \]
    Then, we can compute \emph{offline} an estimator of $\mathrm{D}_{\mathrm{KL}}\left(d_{z}(S,A)||\dE(S,A)\right)$  which is given by
    \[
    \phi_z=\mathbb{E}_{d_{O}(s,a)}\left[\eta_{z}(s,a)\log\frac{\eta_{z}(s,a)}{\etae(s,a)}\right].
    \]
\end{lem}
\begin{proof}
    By Lemma~\ref{app:lem:dz_dEz} we have
    \[
        \mathrm{D}_{\mathrm{KL}}\left(d_{z}(S,A)||\dE(S,A)\right) = \mathrm{D}_{\mathrm{KL}}\left(d_{z}(S,A)||d_{O}(S,A)\right)-\mathbb{E}_{d_{z}(s,a)}\left[\log\frac{\dE(s,a)}{d_{O}(s,a)}\right].
    \]
    For the first term, we have 
    \begin{eqnarray*}
        \mathrm{D}_{\mathrm{KL}}\left(d_{z}(S,A)||d_{O}(S,A)\right) & = & \mathbb{E}_{d_{z}(s,a)}\log\frac{d_{z}(s,a)}{d_{O}(s,a)}\\
        & = & \mathbb{E}_{d_{O}(s,a)}\left[\eta_{z}(s,a)\log\eta_{z}(s,a)\right].
    \end{eqnarray*}
    The second term reduces to
    \[
    \mathbb{E}_{d_{z}(s,a)}\left[\log\frac{\dE(s,a)}{d_{O}(s,a)}\right]=\mathbb{E}_{d_{O}(s,a)}\left[\eta_{z}(s,a)\log\etae(s,a)\right].
    \]
\end{proof}

\begin{lem}[Structural]\label{app:lem:dz_dEz}
    Suppose $0<\eta_{z}(s,a),\etae(s,a)<\infty$ for all $(s,a)\in\mathcal{D}_O$. 
    Then, we have
    \[
    \mathrm{D}_{\mathrm{KL}}\left(d_{z}(S,A)||\dE(S,A)\right)=\mathrm{D}_{\mathrm{KL}}\left(d_{z}(S,A)||d_{O}(S,A)\right)-\mathbb{E}_{d_{z}(s,a)}\left[\log\frac{\dE(s,a)}{d_{O}(s,a)}\right].
    \]
\end{lem}
\begin{proof}
    By definition of KL-divergence, we have
    \begin{eqnarray*}
        \mathrm{D}_{\mathrm{KL}}\left(d_{z}(S,A)||\dE(S,A)\right)&=&\mathbb{E}_{d_{z}(s,a)}\left[\log\left(\frac{d_{z}(s,a)}{d_{O}(s,a)}\cdot\frac{d_{O}(s,a)}{\dE(s,a)}\right)\right]\\&=&\mathrm{D}_{\mathrm{KL}}\left(d_{z}(S,A)||d_{O}(S,A)\right)-\mathbb{E}_{d_{Z}(s,a)}\left[\log\frac{\dE(s,a)}{d_{O}(s,a)}\right].
    \end{eqnarray*}
\end{proof}

\section{Importance Sampling}\label{app:sec:importance_sampling}

\begin{lem}[Importance Sampling]\label{app:lem:IS}
    Given ratios $\eta_{z}(s,a)$, it holds for any function $f(s,a)$ that
    \[
    \mathbb{E}_{d_{z}^{\star}(s,a)}\left[f(s,a)\right]=\mathbb{E}_{d_{O}(s,a)}\left[\eta_{z}(s,a)f(s,a)\right].
    \]
    In particular, for any function $g(s)$ we have
    \[
    \mathbb{E}_{d_{z}^{\star}(s)}\left[g(s)\right]=\mathbb{E}_{d_{O}(s,a)}\left[\eta_{z}(s,a)g(s)\right].
    \]
\end{lem}
\begin{proof}
    The first conclusion follows by definition of $\eta_{z}(s,a)=d_z(s,a)/d_O(s,a)$, whereas the second uses
    \[
        \mathbb{E}_{d_{z}^{\star}(s)}\left[g(s)\right]=\mathbb{E}_{d_{z}^{\star}(s,a)\pi_{z}^{\star}(a|s)}\left[g(s)\right]=\mathbb{E}_{d_{z}^{\star}(s,a)}\left[g(s)\right]=\mathbb{E}_{d_{O}(s,a)}\left[\eta_{z}(s,a)g(s)\right].
    \]
\end{proof}

\subsection{Empirical Estimators}\label{app:finite_sample_estimators}

Recall that the primal optimal solution  satisfies 
\[
\etaz:=\frac{d_{z}^{\star}(s,a)}{d_{O}(s,a)}=\mathrm{softmax}_{d_{O}(s,a)}\left(R_{z}^{\lambda}(s,a)+\gamma\mathcal{T}V_{z}^{\star}(s,a)-V_{z}^{\star}(s)\right),
\]
where 
\begin{equation}\label{app:def:softmax_dO}
\mathrm{softmax}_{p(x)}(g(x))=\frac{\exp\{g(x)\}}{\mathbb{E}_{p(x^{\prime})}[\exp\{g(x^{\prime})\}]}.
\end{equation}
In the rest of this section, we denote the above TD-error term by
\[
    \delta_{z}(s,a)=R_{z}^{\mu}(s,a)+\gamma\mathcal{T}V_{z}^{\star}(s,a)-V_{z}^{\star}(s).
\]
By assumption, the offline dataset $\mathcal{D}_{O}$ is sampled u.a.r. from a state-action occupancy distribution $d_{O}(s,a)$. 
Let $\{w_{z}(s,a)\}_{(s,a)\in\mathcal{D}_{O}}$ be a discrete probability distribution, computed by a softmax, over the offline dataset $\mathcal{D}_{O}$, namely
\[
 w_{z}(s,a)=\mathrm{softmax}_{\mathcal{D}_{O}}(\delta_{z}(s,a))=\frac{\exp\{\delta_{z}(s,a)\}}{\sum_{(s^{\prime},a^{\prime})\in\mathcal{D}_{O}}\exp\{\delta_{z}(s^{\prime},a^{\prime})\}}.
\]
We are now ready to state the main result of this section.

\newpage
\begin{lem}[KL-divergence Estimator]\label{app:lem:KL_div_Estimator}
The following expression
\[
\sum_{(s,a)\in\mathcal{D}_{O}}w_{z}(s,a)\big[\log w_{z}(s,a)-\log w_{\widetilde{E}}(s,a)\big]
\]
is an empirical estimator of the KL-divergence $\mathrm{D}_{\mathrm{KL}}\left(d_{z}(S,A)||\dE(S,A)\right)$.
\end{lem}
\begin{proof}
We estimate the expectation $\mathbb{E}_{d_{O}(s,a)}\exp\{\delta(s,a)\}$ using an empirical estimate $\frac{1}{|\mathcal{D}_{O}|}\sum_{(s,a)\in\mathcal{D}_{O}}\exp\{\delta_{z}(s,a)\}$ over the offline-dataset $\mathcal{D}_{O}$. 
By definition of $\mathrm{softmax}_{d_{O}(s,a)}$, see~\cref{app:def:softmax_dO}, the following expression
\[
\widetilde{\eta}_{z}(s,a)=\frac{\exp\{\delta_z(s,a)\}}{\frac{1}{|\mathcal{D}_{O}|}\sum_{(s^{\prime},a^{\prime})\in\mathcal{D}_{O}}\exp\{\delta_{z}(s^{\prime},a^{\prime})\}}=|\mathcal{D}_{O}|w_{z}(s,a)
\]
is an empirical estimator of the importance weight $\eta_{z}(s,a)$.
Similarly, $\widetilde{\eta}_{\widetilde{E}}(s,a)=|\mathcal{D}_O|w_{\widetilde{E}}(s,a)$ is an estimator of $\eta_{\widetilde{E}}(s,a)$.
Then, the statement follows by combining Lemma~\cref{app:lem:SA-KL-Estimator}, the definition of importance ratios $\eta_{z}(s,a)=d_z(s,a)/d_O(s,a)$, $\etae(s,a)=\dE(s,a)/d_O(s,a)$ and
\begin{eqnarray*}
D_{\mathrm{KL}}\left(d_{z}(S,A)||\dE(S,A)\right)&=&\mathbb{E}_{d_{O}(s,a)}\left[\eta_{z}(s,a)\log\frac{\eta_{z}(s,a)}{\etae(s,a)}\right]\\&\approx&\frac{1}{|\mathcal{D}_{O}|}\sum_{(s,a)\in\mathcal{D}_{O}}\widetilde{\eta}_{z}(s,a)\log\frac{\widetilde{\eta}_{z}(s,a)}{\widetilde{\eta}_{\widetilde{E}}(s,a)}\\&=&\sum_{(s,a)\in\mathcal{D}_{O}}w_{z}(s,a)\log\left(\frac{w_{z}(s,a)}{w_{\widetilde{E}}(s,a)}\right).
\end{eqnarray*}      
\end{proof}

\begin{lem}[Off-Policy Expectation Estimator]\label{app:lem:off-policy:expectation:estimator}
    For any function $f(s,a)$ the following expression
    \[
    \sum_{(s,a)\in\mathcal{D}_{O}}w_{z}(s,a)f(s,a)
    \]
    is an empirical estimator of the expectation $\mathbb{E}_{d_{z}^{\star}(s,a)}\left[f(s,a)\right]$.
\end{lem}
\begin{proof}
    By combining \cref{app:lem:IS} and similar arguments as in the proof of \cref{app:lem:KL_div_Estimator}, we have
    \begin{eqnarray*}
        \mathbb{E}_{d_{z}^{\star}(s,a)}\left[f(s,a)\right]&=&\mathbb{E}_{d_{O}(s,a)}\left[\eta_{z}(s,a)f(s,a)\right]\\&\approx&\frac{1}{|\mathcal{D}_{O}|}\sum_{(s,a)\in\mathcal{D}_{O}}\widetilde{\eta}_{z}(s,a)f(s,a)\\&=&\sum_{(s,a)\in\mathcal{D}_{O}}w_{z}(s,a)f(s,a).
    \end{eqnarray*}
\end{proof}

\newpage
\section{Unconstrained Formulation}\label{app:sec:unconstrained}

SMODICE~\citep{ma2022smodice} minimizes a KL-divergence between the policy state occupancy and the expert state occupancy, expressed as
\begin{eqnarray}
\min_{d(S)} \Dkl \left(d(S)||d_{E}(S)\right).\label{app:eq:smodice}
\end{eqnarray}

A naive approach to extend the above problem formulation to the unsupervised skill discovery setting, is to consider an additional diversity term in the objective.
In particular, adding a scaled mutual information term $\gI(S;Z)$ and maximizing over a set of skill-conditioned state occupancies $\{d_z(S)\}_{z\in Z}$, namely

\begin{equation}\label{app:eq:uc-problem}
\max_{\{d_z(S)\}_{z\in Z}} \alpha \gI(S;Z) - \sum_{z \in Z} \Dkl \left(d_{z}(S)||d_{E}(S)\right).
\end{equation}

Here, the level of diversity is controlled by a hyperparameter $\alpha$.
However, $\alpha$ is arbitrary, and no constraint on closeness to the expert state occupancy is enforced.
We proceed by using the variational lower bound in \cref{eq:VarLB} and assuming a categorical uniform distribution $p(z)$ over the set of latent skills $Z$, which consists of $|Z|$ distinct indicator vectors in $\mathbb{R}^{|Z|}$.
This reduce the optimization problem to
\begin{equation}\label{app:eq:uc-problem}
\max_{{d_{z}(s), \cdiv{q(z|s)}}}\sum_{z\in Z}\left\{ \alpha\mathbb{E}_{d_{z}(s)}\left[\frac{\log\left(\cdiv{q(z|s)}|Z|\right)}{|Z|}\right]-\Dkl\left(d_{z}(S)||d_{E}(S)\right)\right\}.
\end{equation}

\begin{thm}\label{thm:kl-upper-bound} \citep{ma2022smodice} 
Suppose \cref{asm:base} holds. Then, we have
    \[
        \Dkl\left(d_{z}(S)\|d_{E}(S)\right)\le\E_{d_{z}(s)}\left[\log\frac{d_{O}(s)}{d_{E}(s)}\right]+\Dkl(d_{z}(S,A)\|d_{O}(S,A)).
    \]
\end{thm}

By \cref{thm:kl-upper-bound} and linearity of the objective, Problem~\eqref{app:eq:uc-problem} reduces to optimizing separately for each latent skill $z$ the following optimization problem
\begin{equation}\label{final:problem}
    \max_{{d_{z}(s), \cdiv{q(z|s)}}}\mathbb{E}_{d_{z}(s)}\left[R_{z}^{\alpha}(s,a)\right]-\Dkl(d_{z}(S,A)\|d_{O}(S,A)),
\end{equation}
where $R_{z}^{\alpha}(s,a)$ is defined as
\begin{equation}\label{eqn:reward-uc}
R_{z}^{\alpha}(s,a):=
\underbrace{\vphantom{\frac{Z}{\|\lambda_z}}\log\cdem{\frac{d_E(s)}{d_O(s)}}}_{\text{Expert Imitation}} + 
\underbrace{ \alpha 
 \frac{\log\left(\cdiv{q(z|s)}|Z|\right)}{|Z|}}_{\text{Skill Diversity}}.
\end{equation}

The ratios $\cdem{\frac{d_E(s)}{d_O(s)}}$ can be computed by training a discriminator $c(s)$ tasked to distinguish between samples from $d_E(s)$ and $d_O(s)$.
More specifically, since the optimal Bayesian discriminator satisfies $c^{\star}(s)=d_E(s)/(d_E(s)+d_O(s))$, in practice we can use an estimator $c(s)/(1-c(s))\approx \cdem{\frac{d_E(s)}{d_O(s)}}$.

Similar to the \method{}, we can apply the alternating optimization scheme, here with two phases:(i) fixed skill-discriminator (similarly to \cref{eq:LB_fixed_discr_offline}); and 
(ii) fixed importance ratios and policy $\pi_z^{\star}$, where we train the skill-discriminator $\cdiv{q(z|s)}$ (see \cref{app:sec:train-discriminator}).
For the first phase, we use the importance ratios $\eta_z(s,a)$ computed by optimizing the dual-value problem and then applying softmax to the corresponding TD error terms (see \cref{eq:dopt} and \citet{nachum2020reinforcement, ma2022smodice}).

\section{Solo-12 Dataset Collection}\label{app:sec:dataset-collection}

\begin{figure*}[htbp]
    \centering
    \includegraphics[width=0.9\linewidth]{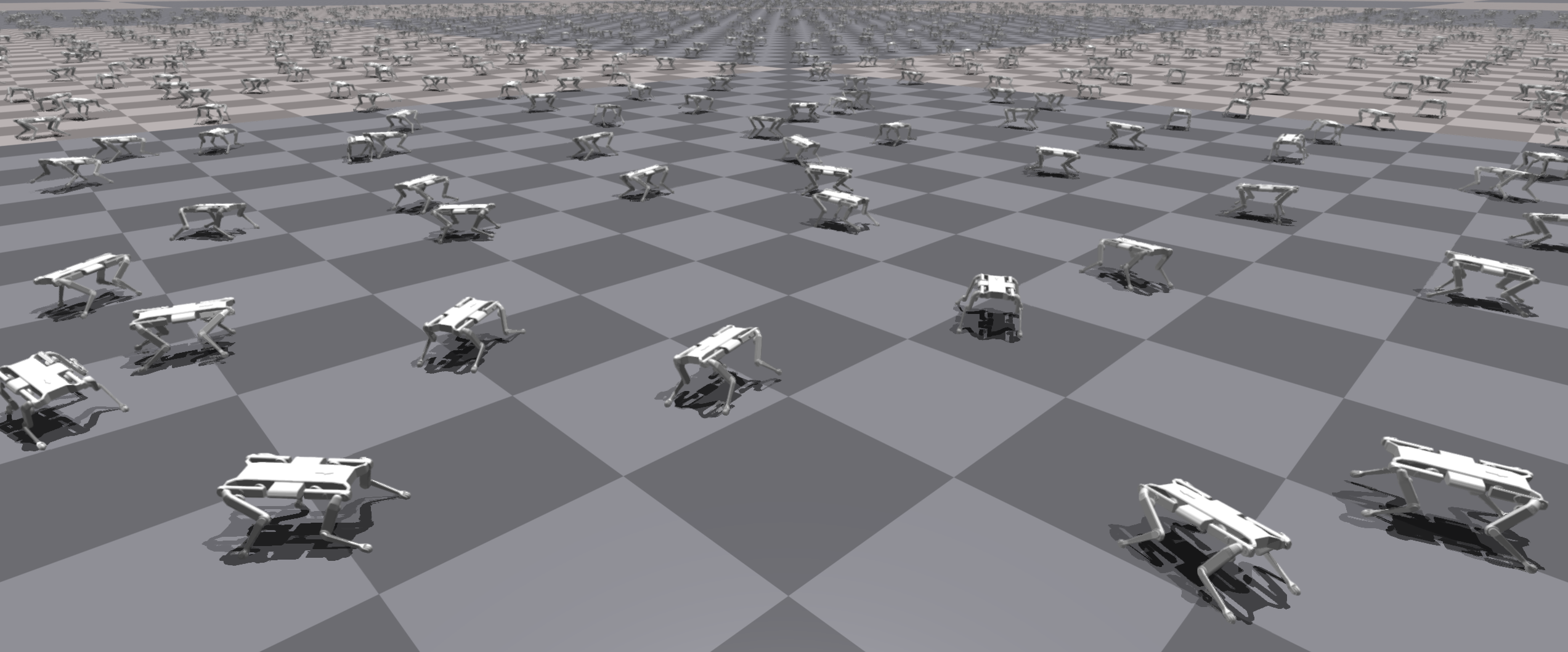}
    \caption{{\small Solo-12 datasets are collected with 4000 environments in parallel using IsaacGym.}}
    \label{fig:isaacgym_collect}
\end{figure*}

As shown in~\cref{fig:isaacgym_collect}, both \emph{expert dataset} and \emph{offline dataset} are collected in parallelized GPU-based environments in Isaac Gym \citep{makoviychuk2021isaac}.
The policies from both locomotion task and obstacle navigation tasks with \envsolo{} are trained using the DOMiNiC~\citep{cheng2024dominic} algorithm to exhibit diverse behaviors while maintaining a certain level of task completion. 
For details on the algorithm used to train the data collection policies, we refer interested readers to \citep{cheng2024dominic}.

\paragraph{Locomotion task.} 
The collecting policies are trained to track randomly sampled velocity commands on the flat ground. 
The state space consists of the linear and angular base velocity vectors, projected gravity vector, joint position, and velocity and commanded velocity.
The actions contain the joint target angles, which will be taken by a PD controller to generate applied torque for each motor.
During collecting, the policies are fed with a fixed forward velocity command of 1 m/s, and zeros for side velocity and yaw rate. 
As mentioned in \cref{sec:experiments}, the policy used for collecting the \emph{expert dataset} is the last and best checkpoint (iteration 2000) and trained without diversity objective, which exhibits a stable mid-height trotting gait pattern. 
The policies for collecting the \emph{offline dataset} are different stochastic checkpoints throughout the training of the skill-conditioned policy. 
The intrinsic reward is designed to maximize the $\ell_2$ distance of the successor features \citep{barreto2016successor} between distinct skills, where in this setting the feature space includes: the base height velocity, base roll and pitch velocities, and feet height velocities. 
The \emph{offline dataset} is composed of 1/2 data from checkpoint 0, 1/4 data from checkpoint 50, 1/8 data from checkpoint 100, 1/16 data from checkpoint 500, 1/32 data from checkpoint 1500 and 1/32 data from checkpoint 2000. 
For each policy checkpoint, we collect data from the 5 corresponding skills, including the target skill.
It is worth noting that more than half of the data from the \emph{offline dataset} comes from the nearly random policies from the start of the training (checkpoints 0 and 50).
Both datasets contain 4000 trajectories with an episode length of 250 steps, or 1 million transitions each.

\paragraph{Obstacle navigation task.}
The policies are trained to track the target position in a terrain of random obstacles of various heights of $\{0.0, 0.05, \dots, 0.25\}$ meters within a fixed time horizon. 
The state space of the agent contains the linear and angular base velocity vectors, projected gravity vector, joint position and velocity, a surrounding height map of the robot and time information, while the actions remain the same as the locomotion task.
During data collection, the policies are tasked with tracking the target 3.0 meter away in the front direction while confronting a $1.0\times1.0$ meter square obstacle of 0.2 meter height. 
The intrinsic reward for training the policy is designed to diversify the base velocity direction such that distinct skills exhibit diverse strategies.
For the \emph{expert dataset}, the used policy is the last and best checkpoint (iteration 2000) trained with diversity objective. 
The \emph{expert dataset} is multi-modal in nature, as the dataset contains diverse strategies for navigating in front of the obstacle, either avoiding it from both sides or climbing it.
On the other hand, the policies for collecting the \emph{offline dataset} are the skill-conditioned checkpoints from iterations $\{0, 50, 100, 150, 200, 250, 500, 1000, 1500, 2000\}$. 
Both datasets contain 2000 trajectories with an episode length of 500 steps, or 1 million transitions each.

\paragraph{Sim-to-Real transfer.} In addition, we use domain randomization during training and data collection, in order to tackle the sim-to-real transfer and to simulate more diverse environment interaction.
Specifically, we randomize the friction coefficient between $[0.5, 1.5]$, additional base mass between $[-0.5, 0.5]$ kg, and simulate the observation noise and an actuator lag of 15 ms. 

\section{SMODICE Expert Return}\label{app:sec:smodice-results}

In \cref{app:tab:smodice-return} we show the performance of the evaluated policies trained by SMODICE\citep{ma2022smodice} on the \envwalker{} and \envhalfcheetah{}.
The results are consistent with the performance that we obtain with \method{} in \cref{fig:d4rl-results}. 
We also note here the importance of having expert state coverage in the offline data that is reflected in the performance of the policies.

\begin{table}[htb]
    \centering
    \begin{tabular}{ccc|c}
\toprule
       Environment  &    dataset    &  $N$   &     $r$ \\
\midrule
halfcheetah & medium-expert & 25  &   81.25 \\
         &        & 50  &   80.47 \\
         &        & 200 &   73.56 \\
         & medium-replay & 25  &   29.28 \\
         &        & 50  &   36.73 \\
         &        & 200 &   60.67 \\
         & random & 25  &   10.89 \\
         &        & 50  &   27.71 \\
         &        & 200 &   78.94 \\
walker2d & medium-expert & 25  &    3.98 \\
         &        & 50  &   19.22 \\
         &        & 200 &    4.10 \\
         & medium-replay & 25  &   15.09 \\
         &        & 50  &    3.60 \\
         &        & 200 &    0.95 \\
         & random & 25  &   52.62 \\
         &        & 50  &  103.52 \\
         &        & 200 &  108.20 \\
\bottomrule
\end{tabular}
    
    \caption{Expected return for SMODICE-learned expert policies in the \envwalker{} and \envant{} environments for $N$ expert trajectories mixed-in.}
    \label{app:tab:smodice-return}
\end{table}

\section{Lagrange Multiplier Stability}\label{app:sec:lagrange-stability}

In \cref{fig:constraints} we observe the behavior of the Lagrange multipliers for different levels of $\epsilon$ for a specific skill $z$ in the \envsolo{} experiment.
In case of $\epsilon \in \{1.0, 2.0\}$, the multipliers fluctuate around a specific level that strikes the balance between diversity and expert imitation.
This can also be validated when observing the violation level in \cref{fig:violation} of the constraint given estimator $\phi_z$, which is for $\epsilon \in \{1.0, 2.0\}$ around $0$.
On the other hand, if we introduce a strong constraint on the KL-divergence ($\epsilon=0.0)$, which is constantly violated, hence $\sigma(\mu_z) = 1$. 
Similarly, if the constraint is too weak, only diversity is optimized, in which case there is a significant degradation in performance (see figure \cref{fig:div-and-task}).

\begin{figure}[htbp]
    \centering
    \legend{}
    \vspace{1em}
    \begin{subfigure}[b]{0.4\textwidth}
        \centering
        \includegraphics[width=\textwidth]{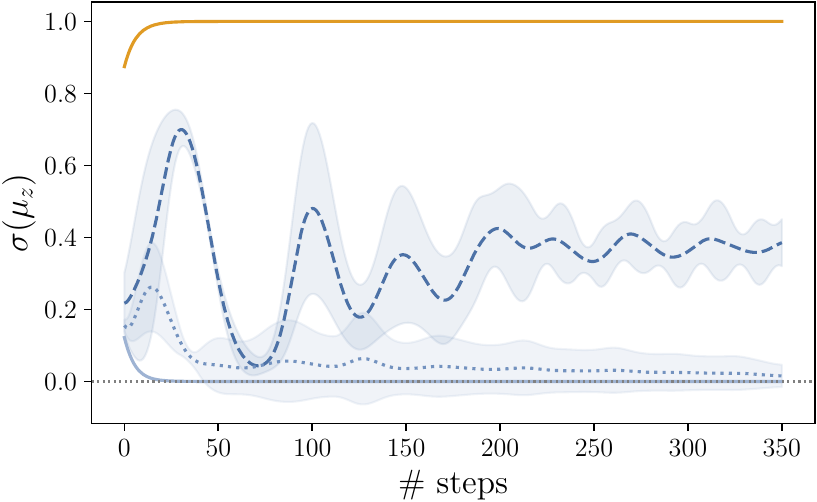}\\
        \caption{}
        \label{fig:lambda_z}
    \end{subfigure}
    \hspace{1em}
    \begin{subfigure}[b]{0.4\textwidth}
        \centering
        \includegraphics[width=\textwidth]{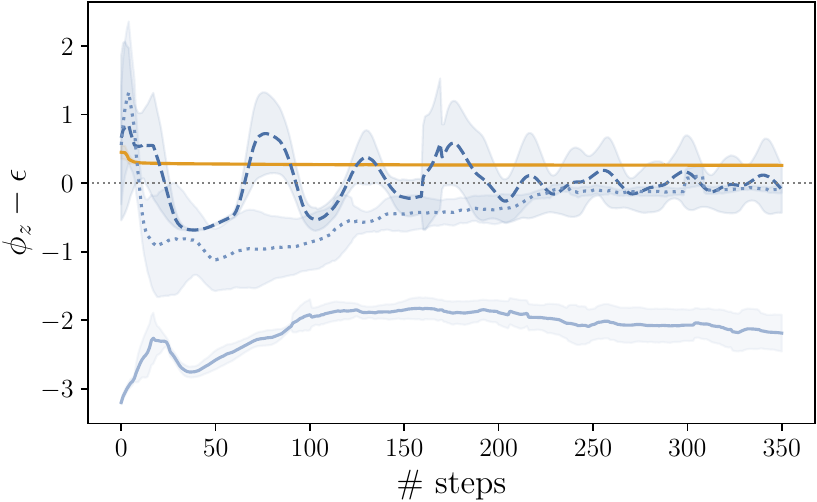}\\
        \caption{}
        \label{fig:violation}
    \end{subfigure}
    \caption{{\small Behavior of Lagrange multipliers. (a)~Evolution of $\sigma(\lambda_z)$ for one skill ($z=1$ chosen arbitrarily), (b) violation of the constraint  for different $\epsilon$. Negative $\phi_z - \epsilon$ indicates no violation.
    Means and standard deviation across restarts.
    The shaded areas show the interval between the 0.25 and 0.75 quantiles, computed over \numSeeds seeds.}}
    \label{fig:constraints}
\end{figure}

In \cref{app:fig:constraints-ant} we show the bounded lagrange multiplier values for three skills and the resulting violations for different $\epsilon$ levels for the \envant{} experiment.
Again, the multiplier values fluctuate around appropriate levels ensuring the the violation of the constraint remains close to 0.

\begin{figure}
    \centering
    \legenddrl{}
    \vspace{1em}
    \begin{subfigure}[b]{0.4\textwidth}
        \centering
        \includegraphics[width=\textwidth]{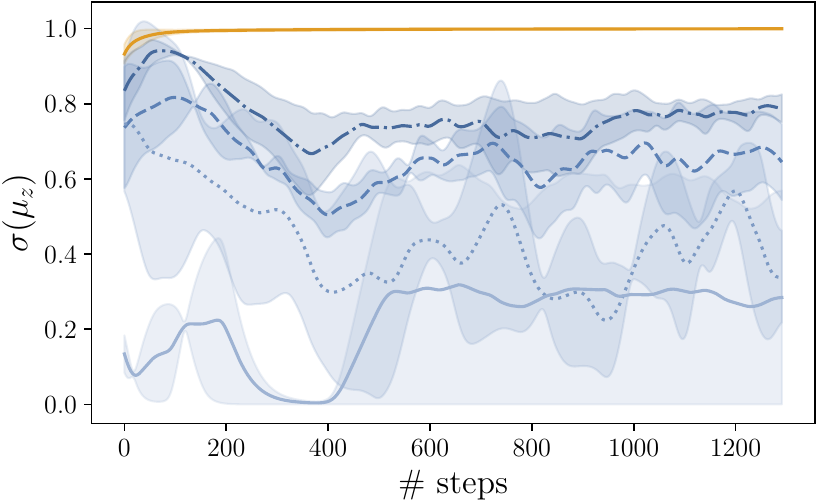}\\
        \caption{}
    \end{subfigure}
    \hspace{1em}
    \begin{subfigure}[b]{0.4\textwidth}
        \centering
        \includegraphics[width=\textwidth]{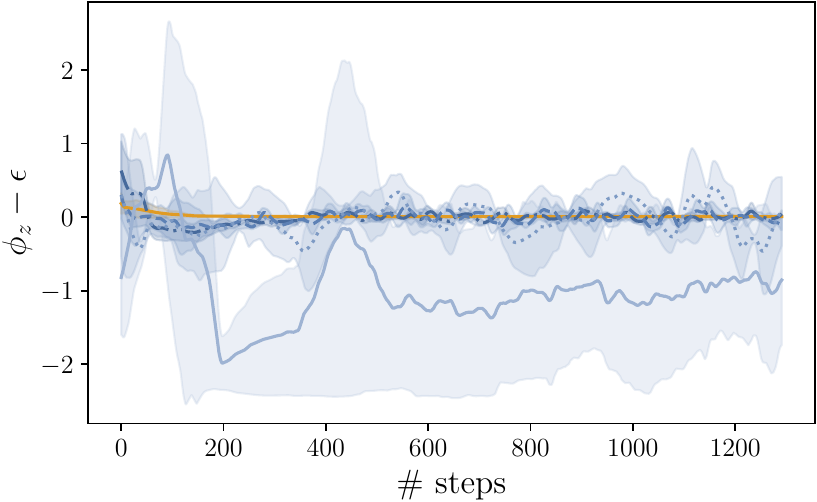}\\
        \caption{}
    \end{subfigure}
    \begin{subfigure}[b]{0.4\textwidth}
        \centering
        \includegraphics[width=\textwidth]{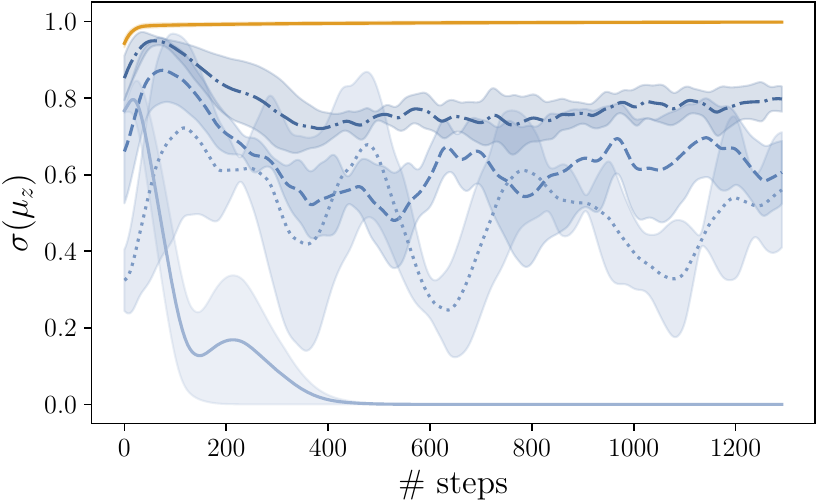}\\
        \caption{}
    \end{subfigure}
    \hspace{1em}
    \begin{subfigure}[b]{0.4\textwidth}
        \centering
        \includegraphics[width=\textwidth]{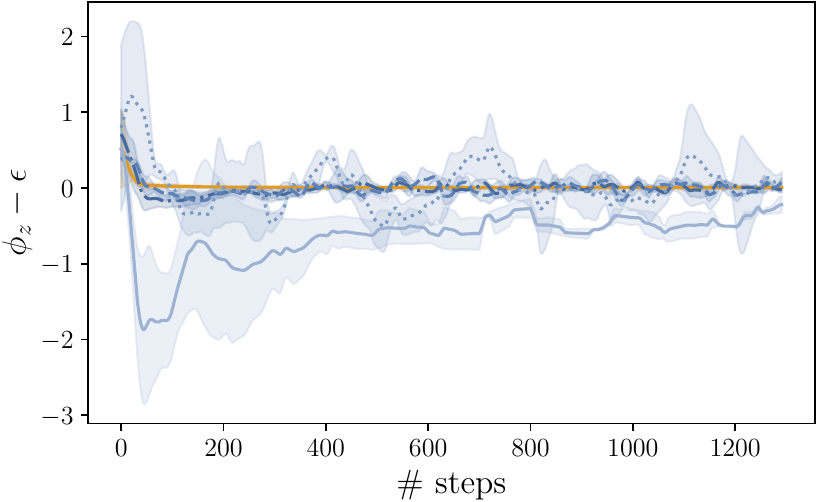}\\
        \caption{}
    \end{subfigure}
    \begin{subfigure}[b]{0.4\textwidth}
        \centering
        \includegraphics[width=\textwidth]{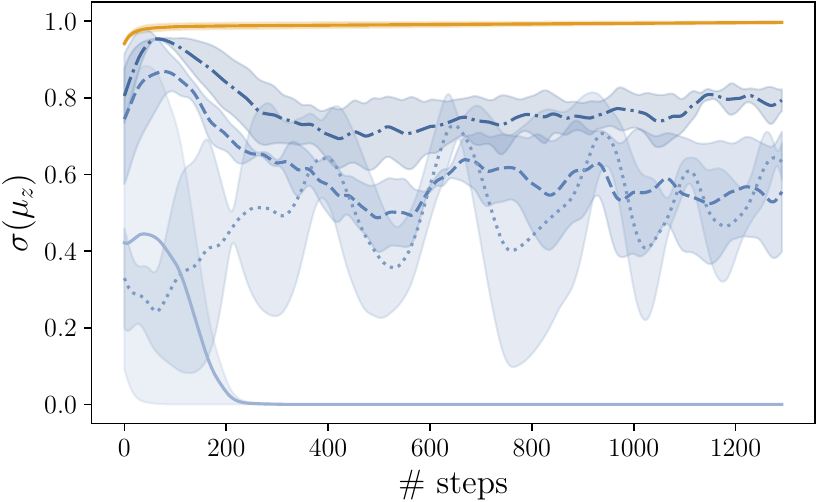}\\
        \caption{}
    \end{subfigure}
    \hspace{1em}
    \begin{subfigure}[b]{0.4\textwidth}
        \centering
        \includegraphics[width=\textwidth]{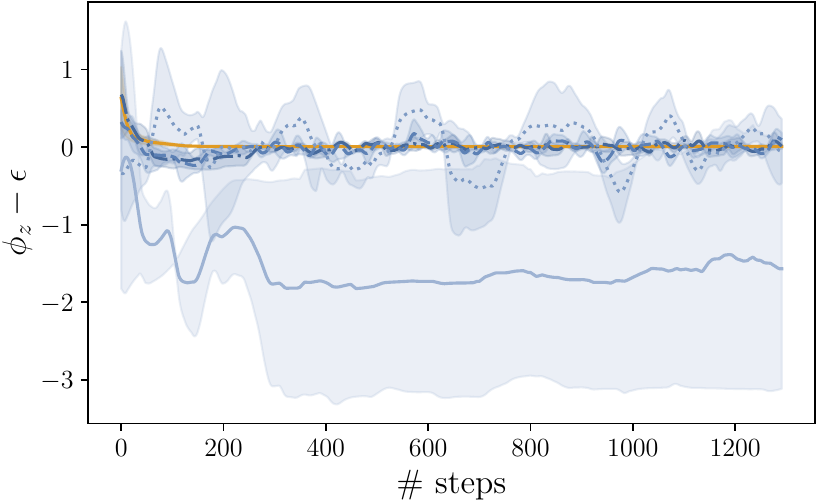}\\
        \caption{}
    \end{subfigure}
    \caption{{\small Behavior of Lagrange multipliers. (a)~Evolution of $\sigma(\lambda_z)$ for one skill ($z=1$ chosen arbitrarily), (b) violation of the constraint  for different $\epsilon$. Negative $\phi_z - \epsilon$ indicates no violation.
    Means and standard deviation across restarts.
    The shaded areas show the interval between the 0.25 and 0.75 quantiles, computed over \numSeeds seeds.}}
    \label{app:fig:constraints-ant}
\end{figure}

\section{Real Robot Deployment}\label{app:sec:real-robot-experiments}

For the locomotion task, we successfully deployed policies exhibiting diverse skills extracted from the \emph{offline dataset} while being able to track a certain velocity similar to the expert on real hardware. 
Our skill-conditioned policy exhibits different walking behaviors with diverse base motions. 
Snapshots of these diverse behaviors can be seen in \cref{fig:skills-on-solo}.

\begin{figure}[htbp]
    \begin{subfigure}[b]{\textwidth}
        \centering
        \includegraphics[width=\textwidth]{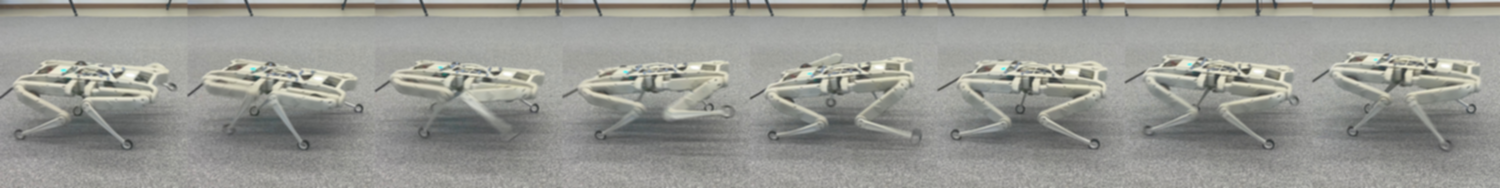}
        \caption{Trot locomotion with wave trunk motion and low base height.}
        \vspace{0.1cm}
    \end{subfigure}
    \begin{subfigure}[b]{\textwidth}
        \centering
        \includegraphics[width=\textwidth]{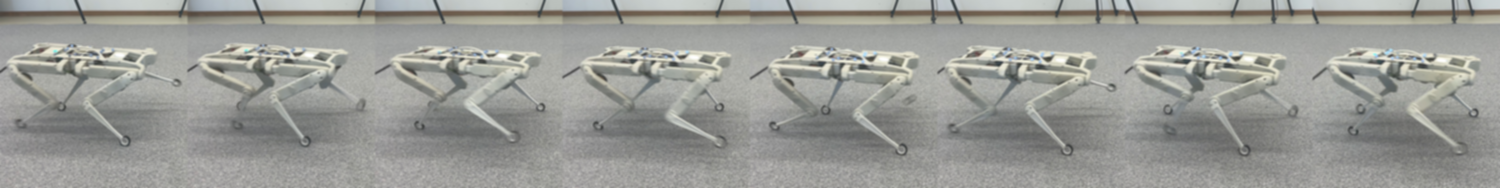}
        \caption{Trot locomotion with middle base height.}
        \vspace{0.1cm}
    \end{subfigure}
    \begin{subfigure}[b]{\textwidth}
        \centering
        \includegraphics[width=\textwidth]{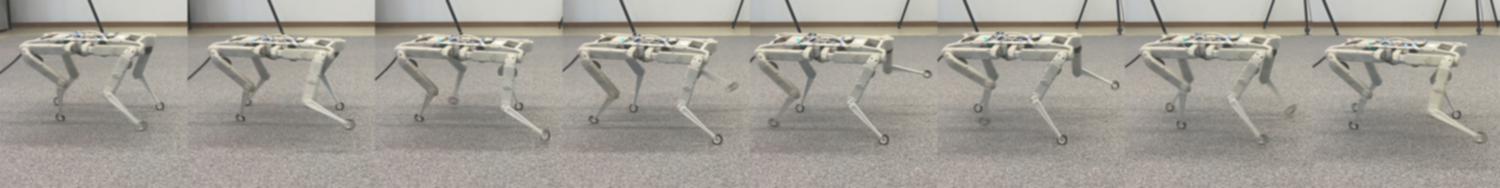}
        \caption{Trot locomotion with high base height.}
        \vspace{0.1cm}
    \end{subfigure}
    \caption{{\small Snapshots of the trained policy exhibiting distinct skills on hardware. From above to bottom, the policy has low, middle and high base positions while moving forward.}}
    \label{fig:skills-on-solo}
\end{figure}

\section{Observation Projection}\label{app:sec:observation-projection}

Imitation learning is of particular interest when the agent's and the target expert policy's state spaces do not necessarily match, but overlap in certain parts, as is often the case when learning from demonstrations.
Our framework naturally accounts for this.
If we consider $\gS'$ to be the state space of the expert and $\gS$ the state space of the agent, we assume that there exists a simple projection mapping $\Pi: \gS' \mapsto \gO$, where $\gO := \{o : o \subset s,  s \in \gS \}$ is the power set of observations, allowing us to potentially imitate beyond expert policies with the same state space as the agent.
Note that the agent still observes its full state $s$, however the projected state $\Pi(s)$ is observed by the expert classifier and skill discriminator.
The projection $\Pi$ can be selected to specify which parts of the state we want to diversify and constrain in terms of occupancy, depending on the task at hand.

\section{Limitations}\label{app:sec:limitations}

The \method{} method also comes with certain caveats. 
Maximizing the mutual information, as a diversity objective, poses a hard optimization problem due to its convexity.
Thus, designing alternative diversity objectives can be beneficial.
Furthermore, closeness in state-action occupancy can be quite restrictive in terms of availability of diverse behaviors that satisfy the constraint. 
Replacing this with constraints on the return of the policy would allow more freedom to optimize diversity in cases where the optimal policy may be multimodal.
The above challenges are promising directions for future work.

\newpage
\section{Robust Obstacle Navigation}\label{app:robustness} 

When the expert data is multi-modal, some modes might be more robust to distribution shift than others.
However, using a uni-modal algorithm such as SMODICE, which tries to match the expert's state occupancy distribution, may not result in a robust policy.
In contrast, each learned DOI skill recovers a particular mode, and as shown in this experiment, at least one DOI skill is robust against a distribution shift.

We consider the task of navigating across a box obstacle to a target position behind it, for the \envsolo{} robot.
For training the DOI skills, we choose the feature vector $\phi(s)$ with linear and angular velocity as input to the skill-discriminator $q(z|\phi(s))$.
The agent used to collect the \emph{expert dataset} can go over or around the box obstacle from the left or right side to reach the target position in the traversable obstacle terrain.
The box has a height of 0.2 meters and a square size of $1.0 \times 1.0$ meters.
As a result, the collected expert data is multi-modal and consists of trajectories over and from the sides of the box obstacle to the target position.

It is important to emphasize that the less direct route to the target position (left or right side of the box) is always the more robust choice, since the agent runs into the risk of slipping or falling while climbing the box.
We evaluate the learned DOI skills and {\color{ourorange} SMODICE expert} on 6 different heights: $\{0.1, 0.2, 0.3, 0.4, 0.5, 0.6\}$ meters. 
The $\{0.3, 0.4, 0.5, 0.6\}$ meters boxes are out-of-distribution and increasingly difficult to traverse from above the box.
In \cref{fig:box_arrow_complete}, we observe the trajectory distributions of the DOI skills and {\color{ourorange} SMODICE expert} collected in simulation.
The arrows indicate the yaw angle of the robot at the trajectory points.

As we can see from the return distributions in \cref{fig:solo12-box-dist}, the performance of the {\color{ourorange} SMODICE expert} is strongly affected by the height of the box, as it is biased towards climbing over the box (this also depends on the initial state of the agent), which becomes increasingly difficult and may not be feasible.
This can be observed from the trajectory distribution shown in the right-most column of \cref{fig:box_arrow_complete}; the trajectories of the {\color{ourorange} SMODICE expert} become increasingly concentrated in front of the box as its height increases.
On the other hand, the three DOI skills (learned with a fixed Lagrange multiplier $\sigma(\mu) = 0.5$) recover diverse behaviors and robustly reach the goal.
Here it is {\color{ourblue} DOI-Skill 3}, which is the most robust in reaching the target position and gives the highest return (see \cref{fig:box_arrow_complete} and \cref{fig:solo12-box-dist}).

In \cref{fig:box_arrow_complete}, each row corresponds to a box with a fixed height $H\in\{0.1, 0.2, 0.3, 0.4, 0.5, 0.6\}$ meters.
Each of the first three columns is associated with a fixed DOI skill ({\color{ourred} red}, {\color{ourgreen} green}, and {\color{ourblue} blue}) and the last column represents the {\color{ourorange} SMODICE expert}.
Each cell shows every 10th step of 60 randomly initialized trajectories, all computed in simulation.
This experiment demonstrates that although {\color{ourorange} SMODICE expert} is multimodal, it gets stuck in front of the box and fails to robustly reach the target position already at a box height of 0.4 meters.
In contrast, the {\color{ourblue} DOI-Skill 3} robustly reaches the target position by bypassing the box from the left side.
The fraction of randomly initialized trajectories stuck in front of the box is significantly smaller for the {\color{ourblue} DOI-Skill 3} than the {\color{ourorange} SMODICE expert}.
This is reflected in the return distribution shown in \Cref{fig:solo12-box-dist}, which has the same row and column structure as \Cref{fig:box_arrow_complete}.

\begin{figure}[h!]
    \centering
    \includegraphics[scale=0.9]{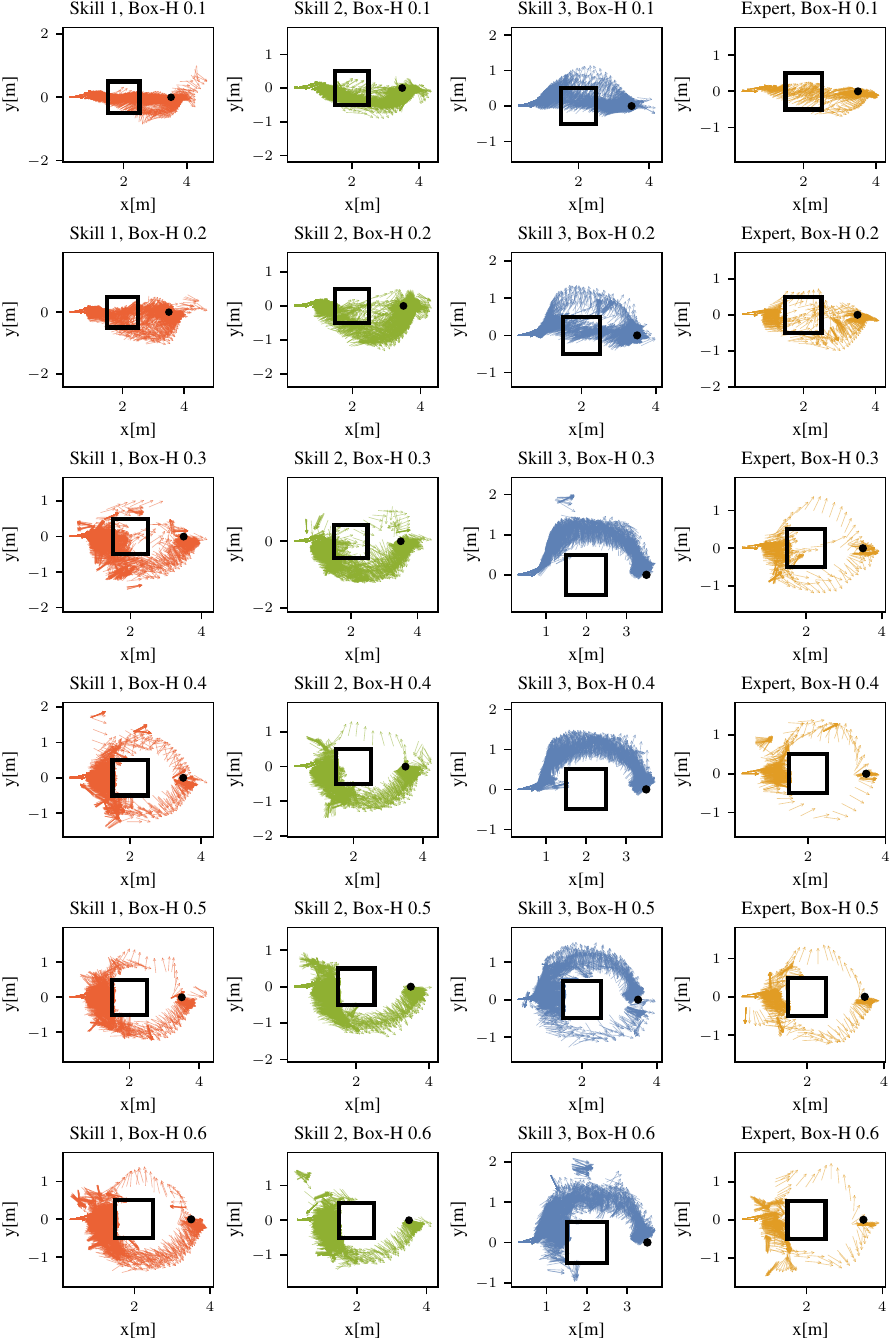}
    \caption{A performance benchmark of the DOI skills and the {\color{darkorange} SMODICE expert} on an obstacle navigation task, where the \envsolo{} is initialized in front of a box and tries to reach a target position behind the box.
    The task consists of six levels of increasing difficulty depending on the height of the box.}
    \label{fig:box_arrow_complete}
\end{figure}

\newlength{\gvs}
\begin{figure}
    \centering
    \includegraphics{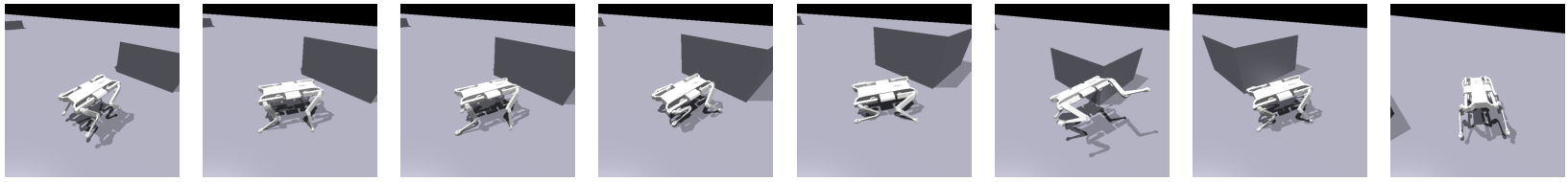}
    
    Skill 1
    
    \vspace{\gvs}    
    \includegraphics{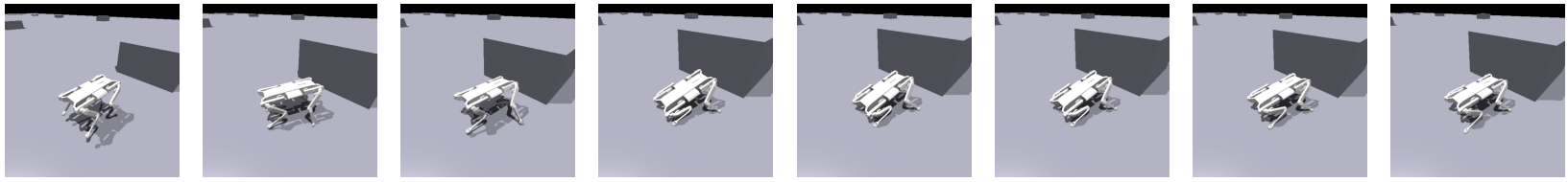}
    
    Skill 2
    
    \vspace{\gvs}
    \includegraphics{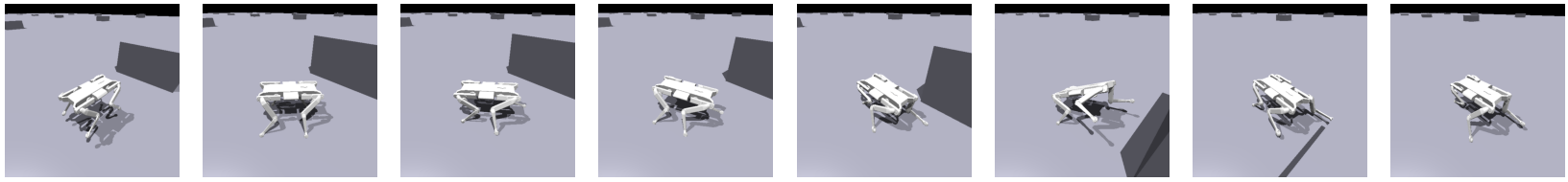}
    
    Skill 3
    
    \vspace{\gvs}
    \caption{Frames from rollout videos of the learned DOI skills for the highest box task, skills 1 and 3 go from the side of the boxes to the goal, and skill 2 reimains in front of the box since it mostly tries to climb it.}
    \label{fig:box-video-frames}
\end{figure}

\begin{figure}
    \centering
    \includegraphics{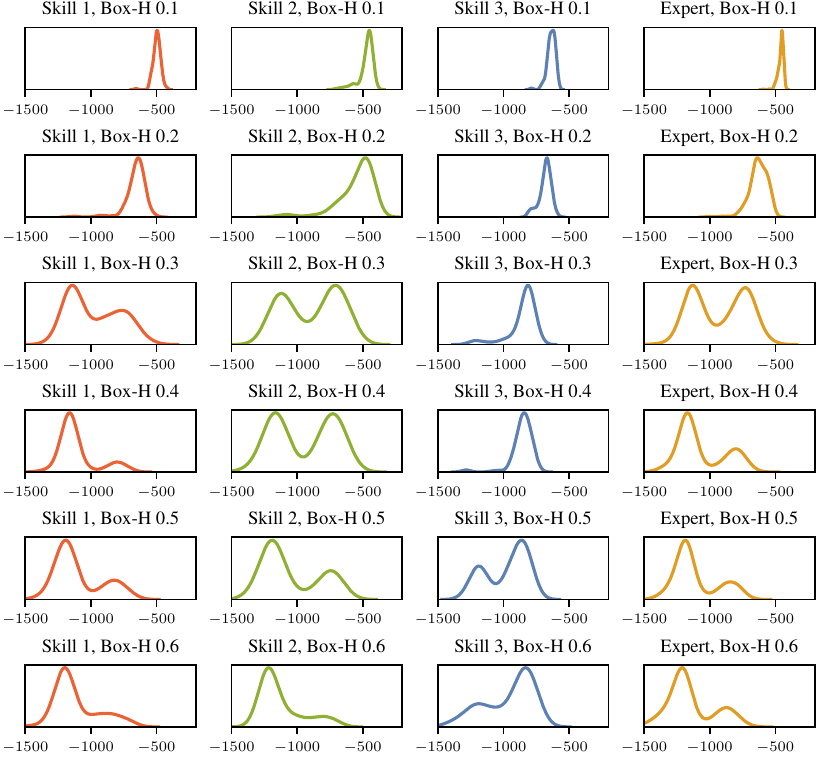}
    \caption{Return distributions for DOI skills and {\color{ourorange} SMODICE}, we see in particular that the SMODICE policy return distribution is greatly affected by increasing the height of the box.}
    \label{fig:solo12-box-dist}
\end{figure}

\section{Additional Experiments}\label{app:sec:additional-experiments}

Instead of learning the Lagrange multipliers $\lambda_z$ via KL estimators $\phi_z$, we can also fix $\lambda_z$ at a certain level, making it a hyperparameter.
In our setting, this also works well, and we demonstrate a tradeoff between diversity and task reward optimization, see \cref{app:fig:div-and-task-fixed,app:fig:constraints-fixed}.
However, in this case we lose the possibility to enforce a certain constraint on the KL-divergence between the skill state-action occupancy and expert state-action occupancy.

\begin{figure}[htbp]
    \centering
    \legendfixed{}\\
    \vspace{1em}
    \begin{subfigure}[b]{0.45\textwidth}
        \centering
        \includegraphics[width=\textwidth]{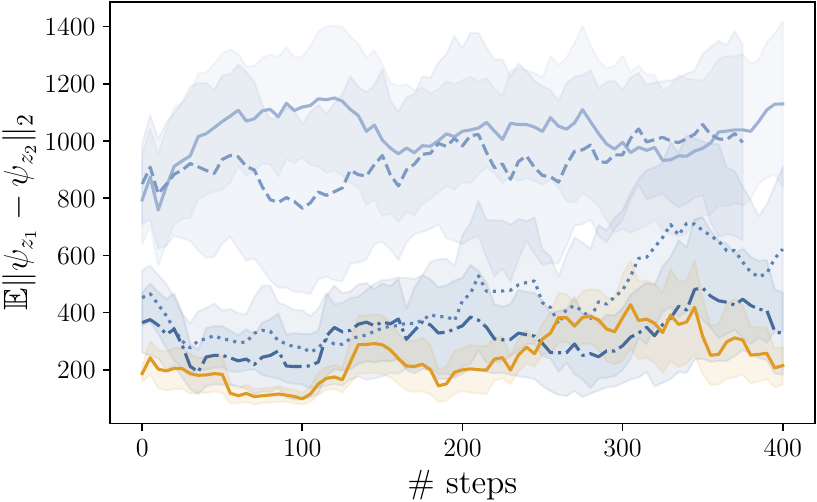}\\
        \caption{}
        \label{app:fig:task-reward-fixed}
    \end{subfigure}
    \hspace{1em}
    \begin{subfigure}[b]{0.45\textwidth}
        \centering
        \includegraphics[width=\textwidth]{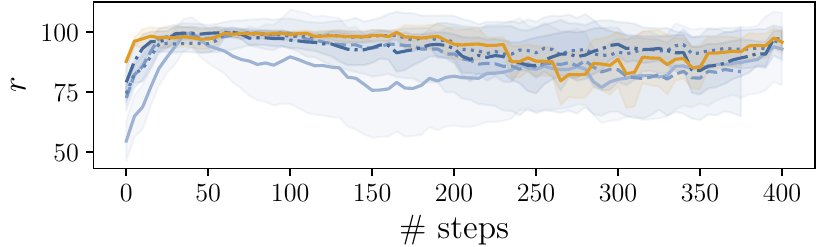}\\
        \includegraphics[width=\textwidth]{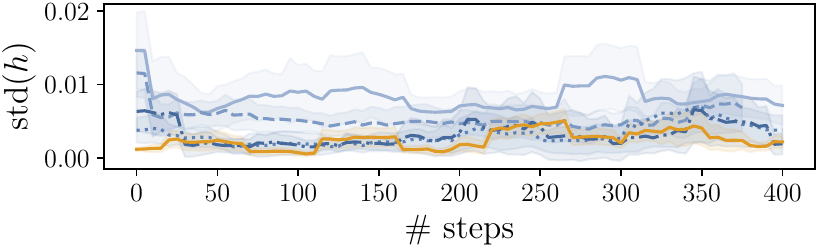}
        \caption{}
        \label{app:fig:height-std-fixed}
    \end{subfigure}
    \caption{{\small (a)~Average $\ell_2$ distance between Monte Carlo estimated successor representations $\psi_z$ of distinct skills,  (b) return $r$ as $\%$ of expert return and standard deviation of base height $\mathrm{std}_z(h)$, depending on a fixed $\sigma(\lambda_z)$ (see legend).
    The shaded areas show the interval between the 0.25 and 0.75 quantiles, computed over \numSeeds seeds.}}
    \label{app:fig:div-and-task-fixed}
\end{figure}

\begin{figure}[htbp]
    \centering
    \legendfixed{}\\
    \vspace{1em}
    \begin{subfigure}[b]{0.45\textwidth}
        \centering
        \includegraphics[width=\textwidth]{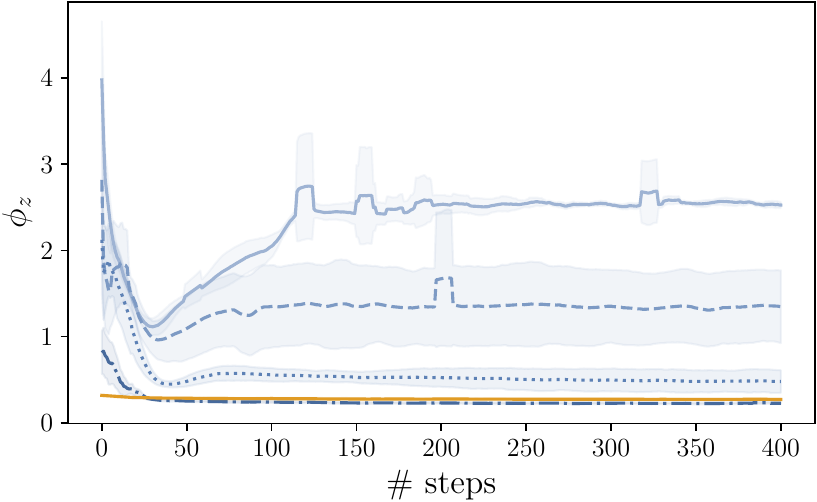}\\
        \caption{}
        \label{app:fig:violation-fixed}
    \end{subfigure}
    \hspace{1em}
    \begin{subfigure}[b]{0.45\textwidth}
        \centering
        \includegraphics[width=\textwidth]{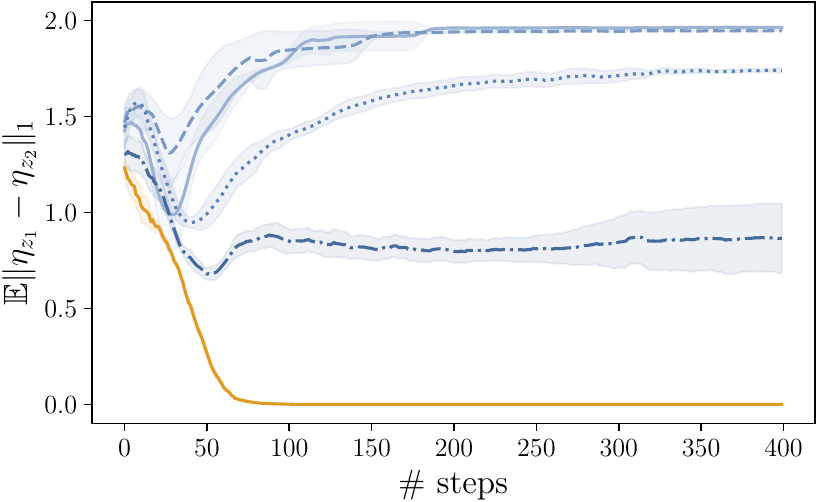}\\
        \caption{}
        \label{app:fig:eta-l1-fixed}
    \end{subfigure}
    \caption{{\small Divergence estimate and $\eta_z$ distance for the case of fixed $\sigma(\lambda_z)$. (a)~Value of divergence estimator $\phi_z$ for a specific skill over the course of training ($z=1$ chosen arbitrarily), (b) average $\ell_1$ distance of $\eta_z$'s of skills.
    Means and standard deviation across restarts.
    The shaded areas show the interval between the 0.25 and 0.75 quantiles, computed over \numSeeds seeds.}}
    \label{app:fig:constraints-fixed}
\end{figure}

We further provide results of applying \method{} to different levels of expert trajectory mix-in to the \emph{medium-replay} and \emph{random} datasets of  \envwalker{} and \envhalfcheetah{} in \cref{tab:walker-additiona-results,tab:halcheetah-additional-results}.

\begin{table}[htb]
    \centering
\begin{tabular}{ccc|ccc}
\toprule
       dataset   &  \# expert mixin   &   $\epsilon$ &     $\mathbb{E} \|\eta_{z_1} - \eta_{z_2}\|$ &            $r$ &       $\mathbb{E} \|\psi_{z_1} - \psi_{z_2}\|$ \\
\midrule
medium-replay & 25  & 0.0 &  0.00 $\pm$ 0.00 &   46.00 $\pm$ 1.46 &   6.16 $\pm$ 0.30 \\
       &     & 0.5 &  0.21 $\pm$ 0.08 &    0.33 $\pm$ 0.48 &   3.54 $\pm$ 2.14 \\
       &     & 1.0 &  1.40 $\pm$ 0.05 &    2.33 $\pm$ 0.51 &   6.09 $\pm$ 2.40 \\
       &     & 2.0 &  1.30 $\pm$ 0.03 &    0.64 $\pm$ 0.11 &   7.67 $\pm$ 4.27 \\
       &     & 4.0 &  1.54 $\pm$ 0.08 &    2.30 $\pm$ 1.64 &  19.26 $\pm$ 2.29 \\
       & 50  & 0.0 &  0.00 $\pm$ 0.00 &   54.29 $\pm$ 2.13 &   5.53 $\pm$ 0.14 \\
       &     & 0.5 &  0.82 $\pm$ 0.28 &   31.31 $\pm$ 7.03 &  14.13 $\pm$ 1.86 \\
       &     & 1.0 &  1.21 $\pm$ 0.15 &    4.33 $\pm$ 0.75 &   0.42 $\pm$ 0.05 \\
       &     & 2.0 &  1.37 $\pm$ 0.03 &    1.61 $\pm$ 0.41 &  13.85 $\pm$ 2.50 \\
       &     & 4.0 &  1.48 $\pm$ 0.12 &    1.11 $\pm$ 0.36 &  22.02 $\pm$ 1.33 \\
       & 200 & 0.0 &  0.00 $\pm$ 0.00 &   98.33 $\pm$ 0.44 &   2.67 $\pm$ 0.26 \\
       &     & 0.5 &  0.45 $\pm$ 0.11 &   74.59 $\pm$ 8.96 &   6.22 $\pm$ 1.17 \\
       &     & 1.0 &  1.20 $\pm$ 0.09 &    2.52 $\pm$ 1.50 &  12.97 $\pm$ 4.33 \\
       &     & 2.0 &  1.30 $\pm$ 0.03 &    2.07 $\pm$ 0.65 &   3.23 $\pm$ 2.02 \\
       &     & 4.0 &  1.59 $\pm$ 0.06 &    1.43 $\pm$ 0.64 &  19.48 $\pm$ 1.43 \\
    \midrule
    random & 25  & 0.0 &  0.00 $\pm$ 0.00 &  36.49 $\pm$ 11.54 &  15.70 $\pm$ 0.48 \\
       &     & 0.5 &  0.93 $\pm$ 0.02 &   20.48 $\pm$ 7.90 &  16.81 $\pm$ 3.14 \\
       &     & 1.0 &  1.30 $\pm$ 0.12 &    3.72 $\pm$ 1.38 &   8.16 $\pm$ 5.43 \\
       &     & 2.0 &  1.45 $\pm$ 0.09 &    1.22 $\pm$ 0.32 &  20.47 $\pm$ 3.08 \\
       &     & 4.0 &  1.27 $\pm$ 0.05 &    0.60 $\pm$ 0.26 &  20.60 $\pm$ 4.17 \\
       & 50  & 0.0 &  0.00 $\pm$ 0.00 &  103.16 $\pm$ 0.69 &   3.32 $\pm$ 0.07 \\
       &     & 0.5 &  1.03 $\pm$ 0.13 &   33.60 $\pm$ 6.64 &  18.27 $\pm$ 2.50 \\
       &     & 1.0 &  1.37 $\pm$ 0.09 &    5.05 $\pm$ 2.66 &  20.16 $\pm$ 3.05 \\
       &     & 2.0 &  1.46 $\pm$ 0.06 &    0.77 $\pm$ 0.29 &  10.46 $\pm$ 3.77 \\
       &     & 4.0 &  1.23 $\pm$ 0.09 &    0.26 $\pm$ 0.11 &  14.33 $\pm$ 1.97 \\
       & 200 & 0.0 &  0.00 $\pm$ 0.00 &  107.43 $\pm$ 0.26 &   1.84 $\pm$ 0.08 \\
       &     & 0.5 &  1.29 $\pm$ 0.07 &  103.29 $\pm$ 1.38 &   6.75 $\pm$ 0.77 \\
       &     & 1.0 &  1.26 $\pm$ 0.22 &    2.43 $\pm$ 0.30 &   7.30 $\pm$ 4.86 \\
       &     & 2.0 &  1.46 $\pm$ 0.10 &    0.47 $\pm$ 0.15 &  15.39 $\pm$ 1.56 \\
       &     & 4.0 &  1.29 $\pm$ 0.01 &    1.91 $\pm$ 0.57 &  19.66 $\pm$ 3.36 \\
\bottomrule
\end{tabular}

    \caption{{\small\envwalker{} metrics across random and medium-replay variants with varying number of mixed-in trajectories of the expert to satisfy the coverage assumption.}}
    \label{tab:walker-additiona-results}
\end{table}

\begin{table}[htb]
    \centering
    \begin{tabular}{ccc|ccc}
    \toprule
          dataset   &  \# expert mixin   &   $\epsilon$ &     $\mathbb{E} \|\eta_{z_1} - \eta_{z_2}\|$ &            $r$ &       $\mathbb{E} \|\psi_{z_1} - \psi_{z_2}\|$ \\
    \midrule
    medium-replay & 25  & 0.0 &  0.00 $\pm$ 0.00 &   37.64 $\pm$ 0.30 &   3.22 $\pm$ 0.06 \\
           &     & 0.5 &  0.83 $\pm$ 0.12 &   36.95 $\pm$ 0.63 &   3.02 $\pm$ 0.10 \\
           &     & 1.0 &  1.36 $\pm$ 0.09 &   24.30 $\pm$ 6.28 &  13.34 $\pm$ 4.84 \\
           &     & 2.0 &  1.44 $\pm$ 0.06 &    6.73 $\pm$ 3.65 &  22.09 $\pm$ 8.15 \\
           &     & 4.0 &  1.27 $\pm$ 0.09 &    2.68 $\pm$ 0.72 &  21.68 $\pm$ 1.87 \\
           & 50  & 0.0 &  0.01 $\pm$ 0.01 &   45.40 $\pm$ 0.22 &   3.26 $\pm$ 0.27 \\
           &     & 0.5 &  1.14 $\pm$ 0.02 &   42.89 $\pm$ 0.19 &   2.94 $\pm$ 0.12 \\
           &     & 1.0 &  1.41 $\pm$ 0.12 &   37.28 $\pm$ 2.41 &   6.18 $\pm$ 1.21 \\
           &     & 2.0 &  1.32 $\pm$ 0.11 &    8.60 $\pm$ 4.66 &  13.66 $\pm$ 1.97 \\
           &     & 4.0 &  1.24 $\pm$ 0.16 &    1.72 $\pm$ 0.18 &  28.74 $\pm$ 7.84 \\
           & 200 & 0.0 &  0.00 $\pm$ 0.00 &   73.60 $\pm$ 0.39 &   3.65 $\pm$ 0.09 \\
           &     & 0.5 &  1.16 $\pm$ 0.08 &   69.91 $\pm$ 1.14 &   3.67 $\pm$ 0.10 \\
           &     & 1.0 &  1.28 $\pm$ 0.13 &  23.74 $\pm$ 12.94 &  13.47 $\pm$ 1.73 \\
           &     & 2.0 &  1.49 $\pm$ 0.10 &   15.52 $\pm$ 4.29 &  32.03 $\pm$ 0.56 \\
           &     & 4.0 &  1.42 $\pm$ 0.07 &    2.16 $\pm$ 0.04 &  11.92 $\pm$ 2.28 \\
    \midrule
    random & 25  & 0.0 &  0.00 $\pm$ 0.00 &    2.80 $\pm$ 0.36 &   5.55 $\pm$ 1.18 \\
           &     & 0.5 &  1.12 $\pm$ 0.04 &    3.03 $\pm$ 0.28 &   4.30 $\pm$ 0.85 \\
           &     & 1.0 &  1.14 $\pm$ 0.12 &    2.24 $\pm$ 0.09 &  10.45 $\pm$ 3.30 \\
           &     & 2.0 &  1.24 $\pm$ 0.08 &    1.73 $\pm$ 0.33 &  25.01 $\pm$ 8.78 \\
           &     & 4.0 &  1.44 $\pm$ 0.03 &    1.60 $\pm$ 0.30 &  35.08 $\pm$ 8.27 \\
           & 50  & 0.0 &  0.00 $\pm$ 0.00 &   31.89 $\pm$ 1.14 &   9.97 $\pm$ 0.58 \\
           &     & 0.5 &  1.14 $\pm$ 0.11 &   10.29 $\pm$ 3.13 &  17.90 $\pm$ 6.01 \\
           &     & 1.0 &  1.42 $\pm$ 0.15 &    6.45 $\pm$ 2.95 &  23.30 $\pm$ 0.96 \\
           &     & 2.0 &  1.41 $\pm$ 0.08 &    2.73 $\pm$ 0.43 &  23.91 $\pm$ 6.98 \\
           &     & 4.0 &  1.68 $\pm$ 0.06 &    1.44 $\pm$ 0.27 &  35.07 $\pm$ 8.08 \\
           & 200 & 0.0 &  0.00 $\pm$ 0.00 &   68.35 $\pm$ 1.25 &   5.20 $\pm$ 0.31 \\
           &     & 0.5 &  1.30 $\pm$ 0.08 &  50.85 $\pm$ 17.30 &   9.80 $\pm$ 3.68 \\
           &     & 1.0 &  1.21 $\pm$ 0.12 &   15.06 $\pm$ 5.58 &  29.57 $\pm$ 4.26 \\
           &     & 2.0 &  1.03 $\pm$ 0.10 &    2.10 $\pm$ 1.99 &  10.84 $\pm$ 7.57 \\
           &     & 4.0 &  1.20 $\pm$ 0.20 &    2.16 $\pm$ 0.05 &  16.90 $\pm$ 5.95 \\
    \bottomrule
    \end{tabular}
    \caption{{\small\envhalfcheetah{} metrics across random and medium-replay variants with varying number of mixed-in trajectories of the expert to satisfy the coverage assumption.}}
    \label{tab:halcheetah-additional-results}
\end{table}

\end{document}